\documentclass{article}

\PassOptionsToPackage{numbers, compress}{natbib}
\usepackage[final]{neurips_2023}
\usepackage{float}
\usepackage{makecell}

\usepackage{amsmath}
\usepackage{amsthm}
\usepackage{multicol}
\usepackage{yhmath}
\newtheorem{theorem}{Theorem}
\newtheorem{assumption}{Assumption}
\newtheorem{definition}{Definition}

\newtheorem{remark}{Remark}
\newtheorem{lemma}{Lemma}
\usepackage{wrapfig}
\usepackage{enumerate}
\usepackage{float}

\usepackage{MnSymbol}
\usepackage{graphicx}
\usepackage{subfigure}
\usepackage[utf8]{inputenc} 
\usepackage[T1]{fontenc}    
\usepackage{hyperref}       
\usepackage{url}            
\usepackage{booktabs}       
\usepackage{amsfonts}       
\usepackage{nicefrac}       
\usepackage{microtype}      
\usepackage{xcolor}         
\usepackage{algorithm}
\usepackage{algorithmicx}
\usepackage{algpseudocode}
\usepackage{multirow}
\usepackage{float}

\title{Offline Imitation Learning with Variational Counterfactual Reasoning}

\author{%
  Zexu Sun$^{\S}$, Bowei He$^{\dagger}$, Jinxin Liu$^{\ddagger}$, Xu Chen$^{\S}$\thanks{Corresponding author}, Chen Ma$^{\dagger}$, Shuai Zhang$^{\P}$   \\
  $^{\S}$Gaoling School of Artificial Intelligence, Renmin University of China\\
  $^{\dagger}$Department of Computer Science, City University of Hong Kong\\
  $^{\ddagger}$School of Engineering, Westlake University\quad $^{\P}$DiDi Chuxing\\
  \texttt{\{sunzexu21, xu.chen\}@ruc.edu.cn}, \texttt{boweihe2-c@my.cityu.edu.hk} \\ \texttt{liujinxin@westlake.edu.cn}, \texttt{chenma@cityu.edu.hk}, \texttt{shuai.zhang@tju.edu.cn}
}

\begin{document}

\maketitle

\begin{abstract}

In offline imitation learning (IL), an agent aims to learn an optimal expert behavior policy without additional online environment interactions. 
However, in many real-world scenarios, such as robotics manipulation, the offline dataset is collected from suboptimal behaviors without rewards. 
Due to the scarce expert data, the agents usually suffer from simply memorizing poor trajectories and are vulnerable to the variations in the environments, lacking the capability of generalizing to new environments.
To automatically generate high-quality expert data and improve the generalization ability of the agent, we propose a framework named \underline{O}ffline \underline{I}mitation \underline{L}earning with \underline{C}ounterfactual data \underline{A}ugmentation (OILCA) by doing counterfactual inference. 
In particular, we leverage identifiable variational autoencoder to generate \textit{counterfactual} samples for expert data augmentation. We theoretically analyze the influence of the generated expert data and the improvement of generalization. 
Moreover, we conduct extensive experiments to demonstrate that our approach significantly outperforms various baselines on both \textsc{DeepMind Control Suite} benchmark for in-distribution performance and \textsc{CausalWorld} benchmark for out-of-distribution generalization. Our code is available at \url{https://github.com/ZexuSun/OILCA-NeurIPS23}.
\end{abstract}

\section{Introduction}\label{sec:intro}

By utilizing the pre-collected expert data, imitation learning (IL) allows us to circumvent the difficulty in designing proper rewards for decision-making tasks and learning an expert policy. 
Theoretically, as long as adequate expert data are accessible, we can easily learn an imitator policy that maintains a sufficient capacity to approximate the expert behaviors~\cite{ross2011reduction,sun2019provably,spencer2021feedback,swamy2021moments}. 
However, in practice, several challenging issues hinder its applicability to practical tasks. In particular, expert data are often limited, and due to the typical requisite for online interaction with the environment, performing such online IL may be costly or unsafe in real-world scenarios such as self-driving or industrial robotics~\cite{kim2016socially,wu2020efficient,zhao2022collective}. 
Alternatively, in such settings, we might instead have access to large amounts of pre-collected unlabeled data, which are of unknown quality and may consist of both good-performing and poor-performing trajectories. 
For example, in self-driving tasks, a number of human driving behaviors may be available; in industrial robotics domains, one may have access to large amounts of robot data. 
The question then arises: can we perform offline IL with only limited expert data and the previously collected unlabeled data, thus relaxing the costly online IL requirements?

Traditional behavior cloning (BC)~\cite{bratko1995behavioural} directly mimics historical behaviors logged in offline data (both expert data and unlabeled data) via supervised learning. 
However, in our above setting, BC suffers from unstable training, as it relies on sufficient high-quality offline data, which is unrealistic. 
Besides, utilizing unlabeled data indiscriminately will lead to severe catastrophes: Bad trajectories mislead policy learning, and good trajectories fail to provide strong enough guidance signals for policy learning.
In order to better distinguish the effect of different quality trajectories on policy learning, various solutions are proposed correspondingly. 
For example, ORIL \cite{zolna2020offline} learns a reward model to relabel previously collected data by contrasting expert and unlabeled trajectories from a fixed dataset; DWBC \cite{xu2022discriminator} introduces a discriminator-weighted task to assist the policy learning. 
However, such discriminator-based methods are prone to overfitting and suffer from poor generalization ability, especially when only very limited expert data are provided.
\begin{wrapfigure}{r}{0.65\textwidth} 
\centering
    \includegraphics[width=1\linewidth]{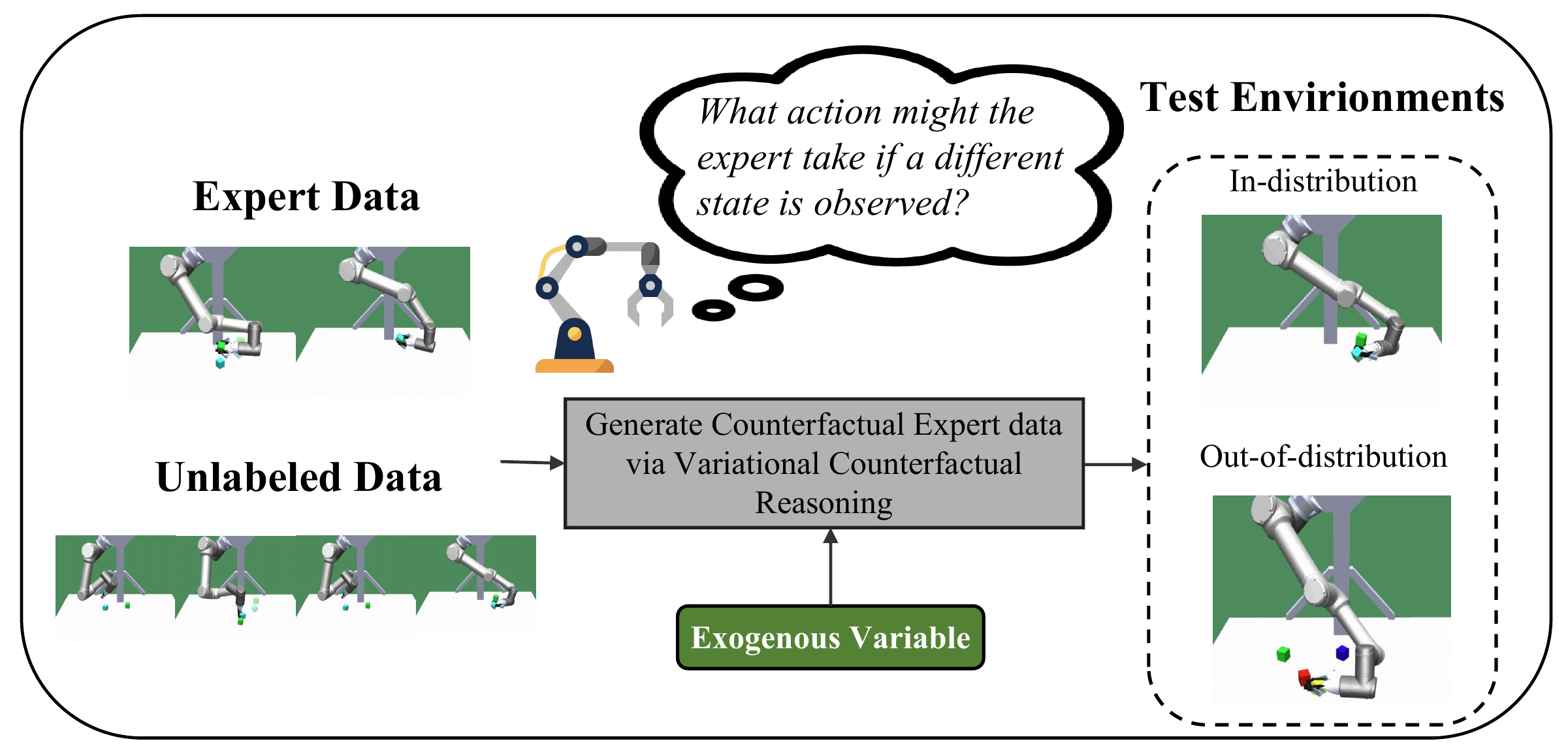}
    \caption{Agent is trained with the collected dataset containing limited expert data and large amounts of unlabeled data, and tested on both in-distribution and out-of-distribution environments.}
    \label{fig:example}
    \vspace{-10pt}
\end{wrapfigure}

Indeed, the expert data play an important role in offline IL and indicate the well-performing agent's intention. Thus, how to fully understand the expert's behaviors or preferences via limited available expert data becomes pretty crucial. In this paper, we propose to investigate counterfactual techniques for interpreting such expert behaviors.
Specifically, we leverage the \textit{variational counterfactual reasoning} \cite{pearl2000models,bottou2013counterfactual} to augment the expert data, as typical imitation learning data augmentation methods \cite{antotsiou2021adversarial,hoshino2021imitation}
easily generate noisy data and inevitably bias the learning agent.
Here, we conduct the counterfactual reasoning of the expert by answering the following question:

\vspace{5pt}
{\centerline{\textit{``What action might the expert take if a different state is observed?''}}}

Throughout this paper, we consider the structure causal model (SCM) underlying the offline training data and introduce an exogenous variable that influences the states and yet is unobserved. 
This variable is utilized as the minimal edits to existing expert training data so that we can generate counterfactual states that cause an agent to change its behavior. 
Intuitively, this exogenous variable captures variations of features in the environment. 
By introducing additional variations in the states during training, we encourage the model to rely less on the idiosyncrasies of a given environment. In detail, we leverage an identifiable generative model to generate the counterfactual expert data, thus enhancing the agent's capabilities of  generalization in test environments (Figure \ref{fig:example}).

The main contributions of this paper are summarized as follows:
\begin{enumerate}[$\bullet$]
    \item We propose a novel learning framework OILCA for offline IL. Using both training data and the augmentation model, we can generate counterfactual expert data and improve the generalization of the learned agent. 

    \item We theoretically analyze the disentanglement identifiability of the constructed exogenous variable and the influence of augmented counterfactual expert data via a sampler policy. We also guarantee the improvement of generalization ability from the perspective of error bound.

    \item We conduct extensive experiments and provide related analysis.
    The empirical results about in-distribution performance on \textsc{DeepMind Control Suite} benchmark and out-of-distribution generalization on \textsc{CausalWorld} benchmark both demonstrate the effectiveness of our method.
\end{enumerate}

\section{Related Works}\label{app:related}

\paragraph{Offline IL} A significant inspiration for this work grows from the offline imitation learning technique on how to learn policies from the demonstrations. Most of these methods take the idea of behavior cloning (BC) \cite{bratko1995behavioural} that utilizes supervised learning to learn to act. However, due to the presence of suboptimal demonstrations, the performance of BC is limited to mediocre levels on many datasets. To address this issue, ORIL \cite{zolna2020offline} learns a reward function and uses it to relabel offline trajectories. However, it suffers from high computational costs and the difficulty of performing offline RL under distributional shifts. Trained on all data, BCND \cite{sasaki2021behavioral} reuses another policy learned by BC as the weight of the original BC objective, but its performance can even be worse if the suboptimal data occupies the major part of the offline dataset. LobsDICE \cite{kim2022lobsdice} learns to imitate the expert policy via optimization in the space of stationary distributions. It solves a single convex minimization problem, which minimizes the divergence between the two-state transition distributions induced by the expert and the agent policy.
CEIL \cite{liu2023ceil} explicitly learns a hindsight embedding function together with a contextual policy. To achieve the expert matching objective for IL, CEIL advocates for optimizing a contextual variable such that it biases the contextual policy towards mimicking expert behaviors. 
DWBC \cite{xu2022discriminator} introduces an additional discriminator to distinguish expert and unlabeled demonstrations, and the outputs of the discriminator serve as the weights of the BC loss. 
CLUE \cite{liu2023clue}  proposes to learn an intrinsic reward that is consistent with the expert intention via enforcing the embeddings of expert data to a calibrated contextual representation. OILCA aims to augment the the scarce expert data to improve the performance of the learned policy.

\paragraph{Causal Dynamics RL} Adopting this formalism allows one to cast several important problems within RL as questions of causal inference, such as off-policy evaluation \cite{buesing2018woulda,oberst2019counterfactual}, learning baselines for model-free RL \cite{mesnard2020counterfactual}, and policy transfer \cite{killian2022counterfactually,lai2023chipformer}. CTRL \cite{lu2020sample} applies SCM dynamics to the data augmentation in continuous sample spaces and discusses the conditions under which the generated transitions are uniquely identifiable counterfactual samples. This approach models state and action
variables as unstructured vectors, emphasizing benefits in modeling action interventions for scenarios such as clinical healthcare where exploratory policies cannot be directly deployed. MOCODA \cite{pitis2022mocoda} applies a learned locally factored dynamics model to an augmented distribution of states and actions to generate counterfactual transitions for RL. FOCUS~\cite{zhu2022offline} can reconstruct the causal structure accurately and illustrate the feasibility of learning causal structure in offline RL. OILCA uses the identifiable generative model to infer the distribution of the exogenous variable in the causal MDP, then performs the counterfactual data augmentation to augment the scarce expert data in offline IL.

\section{Preliminaries}

\subsection{Problem Definition}

We consider the causal Markov Decision Process (MDP)~\cite{lu2020sample} with an additive noise.  
In our problem setting, we have an offline static dataset consisting of \textit{i.i.d} tuples $\mathcal{D}_{\text{all}}=\left\{s_t^i, a_t^i, s_{t+1}^i\right\}_{i=1}^{n_{\text{all}}}$ s.t. $(s_t, a_t) \sim \rho(s_t, a_t), s_{t+1} \sim f_{\varepsilon}(s_t, a_t, u_{t+1})$, where $\rho(s_t,a_t)$ is an offline state-action distribution resulting from some behavior policies, $f_{\varepsilon}(s_t, a_t, u_{t+1})$ represents the causal transition mechanism, $u_{t+1}$ is the sample of the exogenous variable $u$, which is unobserved, and $\varepsilon$ is a small permutation. 
Let $\mathcal{D}_E$ and $\mathcal{D}_U$ be the sets of expert and unlabeled demonstrations respectively, our goal is to only leverage the offline batch data $\mathcal{D}_{\text{all}}:=\mathcal{D}_E \cup \mathcal{D}_U$ to learn an optimal policy $\pi$ without any online interaction. 

\subsection{Counterfactual Reasoning}\label{sec:counter}

We provide a brief background on counterfactual reasoning. Further details can be found in \cite{pearl2000models}.
\begin{definition}[Structural Causal Model (SCM)] A structural causal model $\mathcal{M}$ over variables $\mathbf{X}=$ $\left\{X_1, \ldots, X_n\right\}$ consists of a set of independent exogenous variables $\mathbf{U}=\left\{\mathbf{u}_1, \ldots, \mathbf{u}_n\right\}$ with prior distributions $P\left(\mathbf{u}_i\right)$ and a set of functions $f_1, \ldots, f_n$ such that $X_i=f_i\left(\mathbf{P A}_i, \mathbf{u}_i\right)$, where $\mathbf{P A}_i \subset \mathbf{X}$ are parents of $X_i$. Therefore, the distribution of the SCM, which is denoted $P^{\mathcal{M}}$, is determined by the functions and the prior distributions of exogenous variables.
\end{definition}
Inferring the exogenous random variables based on the observations, we can intervene in the observations and inspect the consequences.
\begin{definition}[\textit{do}-intervention in SCM]\label{def:do} An intervention $I={do}\left(X_i:={f}_i\left(\tilde{\mathbf{P A}}_i, \mathbf{u}_i\right)\right)$ is defined as replacing some functions $f_i\left(\mathbf{P A}_i, \mathbf{u}_i\right)$ with ${f}_i\left(\tilde{\mathbf{P A}}_i, \mathbf{u}_i\right)$, where $\tilde{\mathbf{P A}}_i$ is the intervened parents of $X_i$. The intervened SCM is indicated as $\mathcal{M}^I$, and, consequently, its distribution is denoted as $P^{\mathcal{M} ; I}$.
\end{definition}
The counterfactual inference with which we can answer the \textit{what if} questions will be obtained in the following process:
\begin{enumerate}
    \item Infer the posterior distribution of exogenous variable $P\left(\mathbf{u}_i \mid \mathbf{X}=\mathbf{x}\right)$, where $\mathbf{x}$ is a set of observations. Replace the prior distribution $P\left(\mathbf{u}_i\right)$ with the posterior distribution $P\left(\mathbf{u}_i \mid \mathbf{X}=\mathbf{x}\right)$ in the SCM.
    \item We denote the resulted SCM as $\mathcal{M}_{\mathbf{x}}$ and its distribution as $P^{\mathcal{M}_{\mathbf{x}}}$, perform an intervention $I$ on $\mathcal{M}_{\mathbf{x}}$ to reach $P^{\mathcal{M}_{\mathbf{x}} ; I}$.
    \item Return the output of $P^{\mathcal{M}_{\mathbf{x}} ; I}$ as the counterfactual inference.
\end{enumerate}

\begin{figure}[!t]
    \centering
    \subfigure[SCM of causal MDP]{\includegraphics[width=0.4\linewidth]{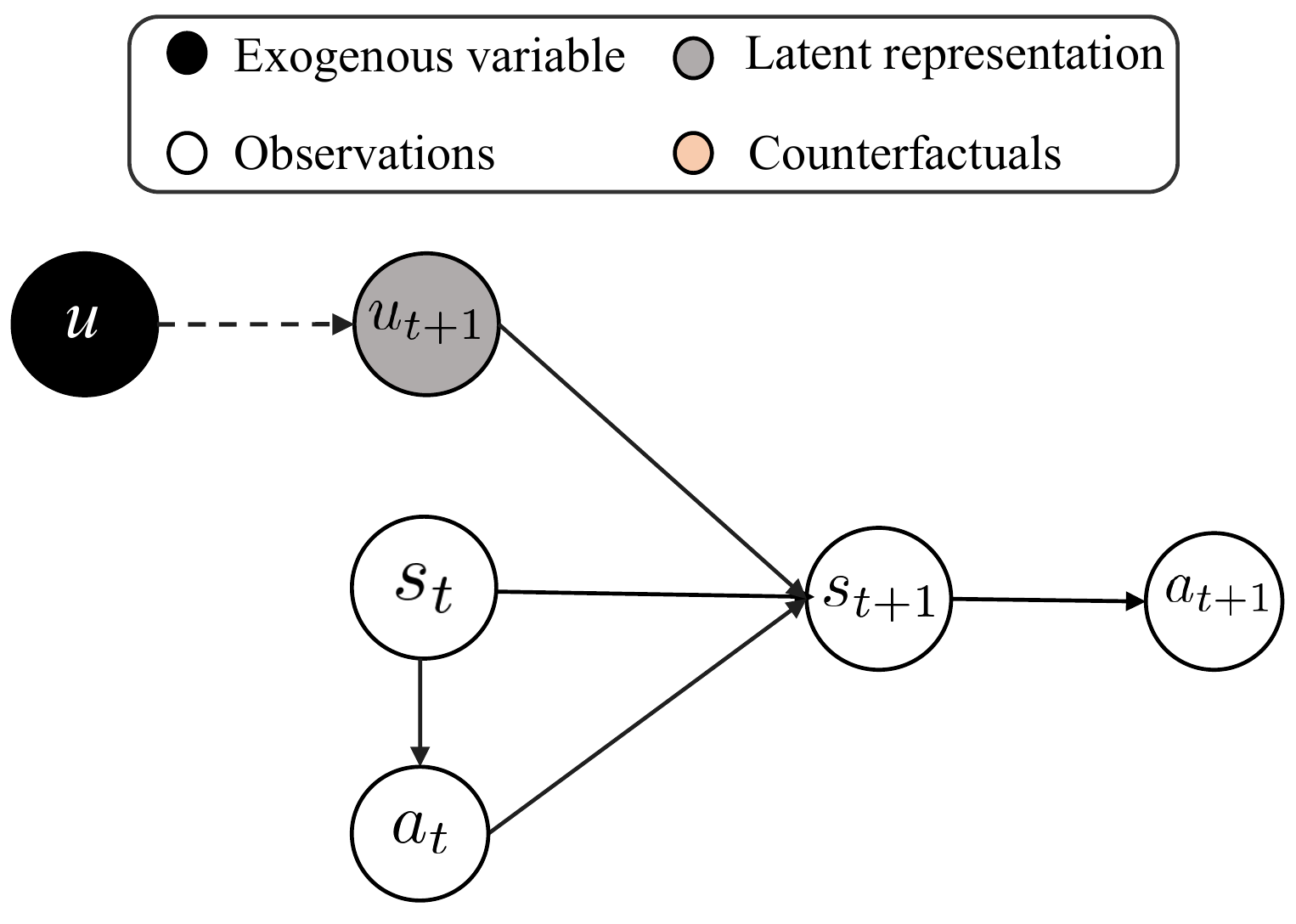} \label{fig:scm_encode}}\hspace{1cm} 
    \subfigure[SCM of causal MDP with \textit{do}-intervention]{\includegraphics[width=0.4\linewidth]{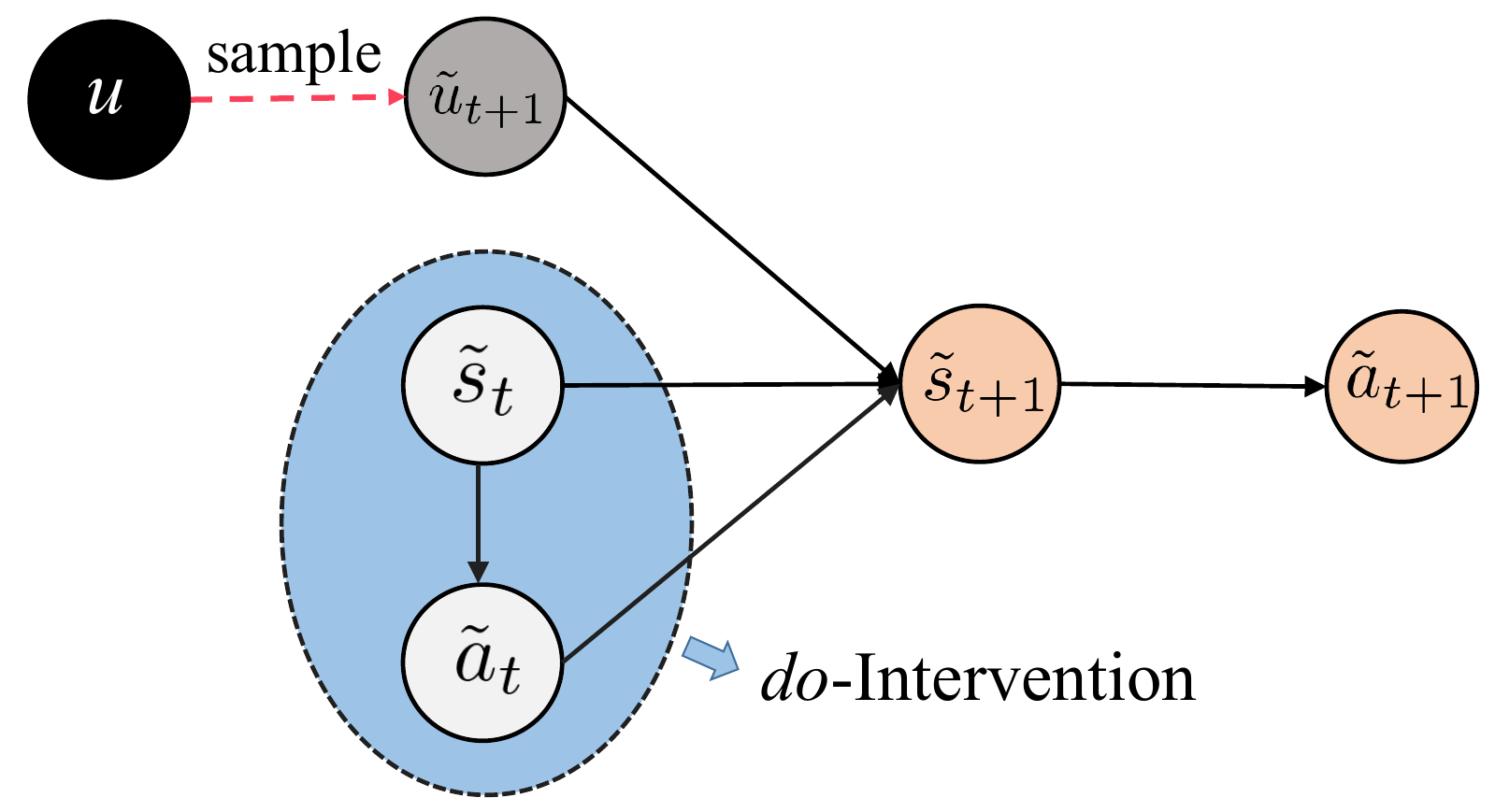}\label{fig:do}}
    \caption{SCM of causal Markov Decision Process (MDP). We incorporate an exogenous variable in the SCM that is learned and utilized for counterfactual reasoning about \textit{do}-intervention.}
    \label{fig:scm}
\end{figure}

\paragraph{SCM representation of causal MDP}We encode causal MDP (under a policy $\pi$) into an SCM $\mathcal{M}$. The SCM shown in Figure \ref{fig:scm} consists of an exogenous variable ${u}$ and a set of functions that transmit the state ${s}_{t}$, action $a_t$ and latent representation $u_{t+1}$ of $u$ to the next state $s_{t+1}$ \textit{e.g.} $s_{t+1} \sim \boldsymbol{f}_{\boldsymbol{\varepsilon}}\left(s_t, a_t, u_{t+1}\right)$, where $\boldsymbol{\varepsilon}$ is a small perturbation, and subsequently, to the next action $a_{t+1}$, \textit{e.g.} $a_{t+1}\sim \pi(\cdot\mid s_{t+1})$. 
To perform counterfactual inference on the SCM $\mathcal{M}$, we intervene in the original parents pair $({s}_t, {a}_t)$ to $(\tilde{s}_t, \tilde{a}_t)$.
Specifically, we sample the latent representation $\tilde{u}_{t+1}$ from the distribution of exogenous variable $u$, and then generate the counterfactual next state $\tilde{s}_{t+1}$, \textit{i.e.} $\tilde{s}_{t+1} \sim \boldsymbol{f}_{\boldsymbol{\varepsilon}}\left(\tilde{s}_t, \tilde{a}_t, \tilde{u}_{t+1}\right)$, and also the counterfactual next action $\tilde{a}_{t+1}$, \textit{i.e.} $\tilde{a}_{t+1} \sim \hat{\pi}(\cdot\mid \tilde{s}_{t+1})$, where $\hat{\pi}$ is the sampler policy.

\subsection{Variational Autoencoder and Identifiability}

We briefly introduce the Variational Autoencoder (VAE) and present its identifiability results. The VAE can be conceptualized as the amalgamation of a generative latent variable model and an associated inference model, both of which are parameterized by neural networks~\cite{kingma2013auto}. 
Specifically, the VAE learns the joint distribution $p_{\boldsymbol{\theta}}(\boldsymbol{x}, \boldsymbol{z})=p_{\boldsymbol{\theta}}(\boldsymbol{x}\mid \boldsymbol{z}) p_{\boldsymbol{\theta}}(\boldsymbol{z})$, where $p_{\boldsymbol{\theta}}(\boldsymbol{x} \mid \boldsymbol{z})$ represents the conditional distribution of observing $\boldsymbol{x}$ given $\boldsymbol{z}$. Here, $\boldsymbol{\theta}$ denotes the set of generative parameters, and $p_{\boldsymbol{\theta}}(\boldsymbol{z})=\Pi_{i=1}^I p_{\boldsymbol{\theta}}\left(z_i\right)$ represents the factorized prior distribution of the latent variables. By incorporating an inference model $q_\phi(\boldsymbol{z} \mid \boldsymbol{x})$, the parameters $\boldsymbol{\phi}$ and $\boldsymbol{\theta}$ can be jointly optimized by maximizing the evidence lower bound (ELBO) on the marginal likelihood $p_\theta(\boldsymbol{x})$.
\begin{equation}
\begin{aligned}
\mathcal{L} & =\mathbb{E}_{q_{\boldsymbol{\phi}}(\boldsymbol{z} \mid \boldsymbol{x})}\left[\log p_{\boldsymbol{\theta}}(\boldsymbol{x} \mid \boldsymbol{z})\right]-D_{\mathrm{KL}}\left(q_{\boldsymbol{\phi}}(\boldsymbol{z} \mid \boldsymbol{x}) \| p(\boldsymbol{z})\right) \\
& =\log p_{\boldsymbol{\theta}}(\boldsymbol{x})-D_{\mathrm{KL}}\left(q_{\boldsymbol{\phi}}(\boldsymbol{z} \mid \boldsymbol{x}) \| p_{\boldsymbol{\theta}}(\boldsymbol{z} \mid \boldsymbol{x})\right) \leq \log p_{\boldsymbol{\theta}}(\boldsymbol{x}),
\end{aligned}
\end{equation}
where $D_{\mathrm{KL}}$ denotes the KL-divergence between the approximation and the true posterior, and $\mathcal{L}$ is a lower bound of the marginal likelihood $p_{\boldsymbol{\theta}}(\boldsymbol{x})$ because of the non-negativity of the KL-divergence.
Recently, it has been shown that VAEs with unconditional prior distributions $p_{\boldsymbol{\theta}}(\boldsymbol{z})$ are not identifiable \cite{locatello2019challenging}, but the latent factors $\boldsymbol{z}$ can be identified with a conditionally factorized prior distribution $p_{\boldsymbol{\theta}}(\boldsymbol{x} \mid \boldsymbol{z})$ over the latent variables to break the symmetry~\cite{khemakhem2020variational}.


\section{Offline Imitation Learning with Counterfactual Data Augmentation}

At a high level, OILCA consists of following steps: (1) data augmentation via variational counterfactual reasoning, and (2) cooperatively learning a discriminator and a policy by using $\mathcal{D}_U$ and augmented $\mathcal{D}_E$. In this section, we detail these two steps and present our method's theoretical guarantees.

\subsection{Counterfactual Data Augmentation}
Our intuition is that the counterfactual expert data can model deeper relations between states and actions, which can help the learned policy be generalized better. Hence, in this section, we aim to generate counterfactual expert data by answering the \textit{what if} question in Section \ref{sec:intro}. Consequently, counterfactual expert data augmentation is especially suitable for some real-world settings, where executing policies during learning can be too costly or slow.

As presented in Section \ref{sec:counter}, the counterfactual data can be generated by using the posterior of the exogenous variable obtained from the observations. Thus, we can briefly introduce the conditional VAE \cite{kingma2013auto,sohn2015learning} to build the latent representation of the exogenous variable $u$. Moreover, Considering the identifiability of unsupervised disentangled representation learning \cite{locatello2019challenging}, an additionally observed variable $c$ is needed, where $c$ could be, for example, the time index or previous data points in a time series, some kind of (possibly noisy) class label, or another concurrently observed variable \cite{khemakhem2020variational}. Formally, let $\boldsymbol{\theta}=(\boldsymbol{f}, \boldsymbol{T}, \boldsymbol{\lambda})$ be the parameters of the following conditional generative model:
\begin{equation}\label{eq:model}
p_{\boldsymbol{\theta}}(s_{t+1},  u \mid s_t, a_t, c)=p_{\boldsymbol{f}}(s_{t+1} \mid   s_t, a_t, u) p_{\boldsymbol{T}, \boldsymbol{\lambda}}(u \mid  c), 
\end{equation}
where we first define:
\begin{equation}\label{eq:exo_noise}
p_{\boldsymbol{f}}(s_{t+1} \mid  s_t, a_t, u)=p_{\boldsymbol{\varepsilon}}(s_{t+1}-\boldsymbol{f}_{s_t,a_t}(u)).
\end{equation}
Equation \eqref{eq:model} describes the generative mechanism of $s_{t+1}$ given the underlying exogenous variable $u$, along with $(s_t, a_t)$. Equation \eqref{eq:exo_noise} implies that the observed representation $s_{t+1}$ is an additive noise function, \textit{i.e.}, $s_{t+1}=\boldsymbol{f}_{s_t,a_t}(u)+\boldsymbol{\varepsilon}$ where $\boldsymbol{\varepsilon}$ is independent of $\boldsymbol{f}_{s_t,a_t}$ or $u$. Moreover, this formulation also can be found in the SCM representation of causal MDP in Section \ref{sec:counter}, where we treat the function $f$ as the parameter $\boldsymbol{f}$ of the model.

Following standard conditional VAE \cite{sohn2015learning} derivation, the evidence lower bound (ELBO) for each sample for $s_{t+1}$ of the above generative model can be written as:
\begin{equation}\label{eq:elbo}
\begin{aligned} \log p(s_{t+1} \mid s_t, a_t, c)  \geq&\mathcal{L}(\boldsymbol{\theta}, \boldsymbol{\phi}):= \log p_{\boldsymbol{f}}(s_{t+1} \mid s_t, a_t,  u)+\log p(c)-\\ &  D_{\mathrm{KL}}(q_{\boldsymbol{\phi}}(u \mid s_{t}, a_{t}, s_{t+1}, c) \| p_{\boldsymbol{T}, \boldsymbol{\lambda}}(u \mid c)) .\end{aligned} 
\end{equation}

Note that the ELBO in Equation \eqref{eq:elbo} contains an additional term of the $\log p(c)$ that does not affect identifiability but improves the estimation for the conditional prior \cite{mita2021identifiable}, we present the theoretical guarantees of disentanglement identifiability in Section \ref{sec:ind}. Moreover, the encoder $q_{\boldsymbol{\phi}}$ contains all the observable variables. The reason is that, as shown in Figure \ref{fig:scm_encode}, from a causal perspective, $s_{t}, a_{t}$ and $u$ form a collider at $s_{t+1}$, which means that when given $s_{t+1}$,  $s_{t}$ and $a_{t}$ are related to $u$.

We seek to augment the expert data in $\mathcal{D}_E$. As shown in Figure \ref{fig:do}, once the posterior of exogenous variable $q_{\boldsymbol{\phi}}\left(u \mid s_t, a_t, s_{t+1}, c\right)$ is obtained, we can perform \textit{do}-intervention. 
In particular, we intervene in the parents $(s_t, a_t)$ of $s_{t+1}$ in $\mathcal{D}_E$ by re-sampling a different pair $(\tilde{s}_t, \tilde{a}_t)$ from the collected data in $\mathcal{D}_{U}$. 
Moreover, we sample the latent representation $\tilde{u}_{t+1}$ from the learned posterior distribution of exogenous variable $u$. 
Then utilizing $\tilde{s}_t$, $\tilde{a}_t$, and $\tilde{u}_{t+1}$, we can generate the counterfactual next state according to Equation \eqref{eq:exo_noise}, which is denoted as $\tilde{s}_{t+1}$. 
Subsequently, we pre-train a sampler policy $\hat{\pi}_E$ with the original $\mathcal{D}_E$ to sample the counterfactual next action $\tilde{a}_{t+1}=\hat{\pi}_E(\cdot\mid \tilde{s}_{t+1})$. 
Thus, for all the generated tuples $(\tilde{s}_{t+1},\tilde{a}_{t+1})$, we constitute them in $\mathcal{D}_E$.

\subsection{Offline Agent Imitation Learning}
In general, the counterfactual data augmentation of our method can enhance many offline IL methods. 
It is worth noting that as described in \cite{zolna2021task}, discriminator-based methods are easy to be over-fitted, which can more directly show the importance and effectiveness of the counterfactual augmented data. Thus, in this section, also to best leverage the unlabeled data, we use Discriminator-Weighted Behavioral Cloning (DWBC) \cite{xu2022discriminator}, a state-of-the-art discriminator-based offline IL method. This method introduces a unified framework to learn the policy and discriminator cooperatively. The discriminator training gets information from the policy $\pi_{\boldsymbol{\omega}}$ as additional input, yielding a new discriminating task whose learning objective is as follows:
\begin{equation}
\begin{aligned}
\mathcal{L}_{\boldsymbol{\psi}}(\mathcal{D}_E, \mathcal{D}_U)=&\eta \underset{(s_t, a_t) \sim \mathcal{D}_E}{\mathbb{E}}[-\log D_{\boldsymbol{\psi}}(s_t,a_t,\log \pi_{\boldsymbol{\omega}})]+\underset{(s'_t, a'_t) \sim \mathcal{D}_U}{\mathbb{E}}[-\log (1-D_{\boldsymbol{\psi}}(s'_t,a'_t,\log \pi_{\boldsymbol{\omega}})] \\
&-\eta \underset{(s_t, a_t) \sim \mathcal{D}_E}{\mathbb{E}}[-\log (1-D_{\boldsymbol{\psi}}(s_t,a_t,\log \pi_{\boldsymbol{\omega}}))],
\end{aligned}\label{eq:discriminator}
\end{equation}
where $D_{{\boldsymbol{\psi}}}$ is the discriminator, $\eta$ is called the class prior. In the previous works \cite{xu2021positive,zolna2020offline}, $\eta$ is a fixed hyperparameter often assigned as 0.5. 

Notice that now $\pi_{\boldsymbol{\omega}}$ appears in the input of $D_{{\boldsymbol{\psi}}}$, which means that imitation information from $\log\pi_{\boldsymbol{\omega}}$ will affect $\mathcal{L}_{{\boldsymbol{\psi}}}$, and further impact the learning of $D_{{\boldsymbol{\psi}}}$. Thus, inspired by the idea of adversarial training, DWBC \cite{xu2022discriminator} introduces a new learning objective for BC Task:
\begin{equation}
\begin{aligned}
\mathcal{L}_{\pi} =& \alpha \underset{(s_t, a_t) \sim \mathcal{D}_E}{\mathbb{E}}[-\log \pi_{\boldsymbol{\omega}}(a_t \mid s_t)]-\underset{(s_t, a_t) \sim \mathcal{D}_E}{\mathbb{E}}\left[-\log \pi_{\boldsymbol{\omega}}(a_t \mid s_t) \cdot \frac{\eta}{d(1-d)}\right]\\&+\underset{(s'_t, a'_t) \sim \mathcal{D}_U}{\mathbb{E}}\left[-\log \pi_{\boldsymbol{\omega}}(a'_t \mid s'_t) \cdot \frac{1}{1-d}\right], \quad \alpha>1.    
\end{aligned}\label{eq:policy}
\end{equation}
where $d$ represents $D_{\boldsymbol{\psi}}(s_t, a_t, \log \pi_{\boldsymbol{\omega}})$ for simplicity, $\alpha$ is the weight factor $(\alpha > 1)$. The detailed training procedure of OILCA is shown in Algorithm \ref{alg:total}. Moreover, we also present more possible combinations with other offline IL methods and the related results in Appendix \ref{app:result}.

\begin{algorithm}[!t]
  \caption{Training procedure of OILCA.}
  \label{alg:total}
\begin{algorithmic}[1]
 \Require Dataset $\mathcal{D}_E$, $\mathcal{D}_U$, $\mathcal{D}_{\text{all}}$, pre-trained sampler policy $\hat{\pi}_E$, hyperparameters $\eta$, $\alpha$, initial variational couterfactual parameters $\boldsymbol{\theta}, \boldsymbol{\phi}$, discriminator parameters ${\boldsymbol{\psi}}$, policy parameters ${\boldsymbol{\omega}}$, data augmentation batch number $B$.
 \Ensure Learned policy parameters ${\boldsymbol{\omega}}$.
\While{counterfactual training} \Comment{Variational counterfactual reasoning}
 \State Sample $(s_t, a_t, s_{t+1},c)\sim \mathcal{D}_\text{all}$ to form a training batch
 \State Update $\boldsymbol{\theta}$ and $\boldsymbol{\phi}$ according to Equation \eqref{eq:elbo}
  \EndWhile
\For{$b=1$ to $B$}\Comment{Expert data augmentation}
\State Sample $(s_t, a_t, s_{t+1},c)\sim \mathcal{D}_E$, $(\tilde{s}_t, \tilde{a}_t) \sim \mathcal{D}_{U}$ to form an augmentation batch
\State Generate the counterfactual $\tilde{s}_{t+1}$ according to Equation \eqref{eq:model}, then predict $\tilde{a}_{t+1}= \hat{\pi}_E(\cdot\mid \tilde{s}_{t+1})$ 
\State $\mathcal{D}_E \cup (\tilde{s}_{t+1}, \tilde{a}_{t+1})$ 
\EndFor
\While{agent training} \Comment{Learning the discriminator and policy cooperatively}
\State Sample $(s_t,a_t) \sim \mathcal{D}_E$ and $(s'_t,a'_t) \sim \mathcal{D}_U$  to form a training batch
  \State Update ${\boldsymbol{\psi}}$ according Equation \eqref{eq:discriminator} every 100 training steps \Comment{Discriminator learning} 
  \State Update $\boldsymbol{\omega}$ according to Equation \eqref{eq:policy}  every 1 training step \Comment{Policy learning}
  \EndWhile
\end{algorithmic}
\end{algorithm}

\subsection{Theoretical Analysis}\label{sec:ind}

In this section, we theoretically analyze our method, which mainly contains three aspects: (1) disentanglement identifiability, (2) the influence of the augmented data, and (3) the generalization ability of the learned policy.

Our disentanglement identifiability extends the theory of iVAE \cite{khemakhem2020variational}. To begin with, some assumptions are needed. Considering the SCM in Figure \ref{fig:scm_encode}, we assume the data generation process in Assumption \ref{ass:data_gen}.

\begin{assumption}\label{ass:data_gen}
(a) The distribution of the exogenous variable $u$ is independent of time but dependent on the auxiliary variable $c$. (b) The prior distributions of exogenous variable $u$ are different across auxiliary variable $c$.
(c) The trasition mechanism $p(s_{t+1}\mid s_t,a_t,u)$ are invariant across different auxiliary variable $c$. (d) Given the exogenous variable $u$, the next state $s_{t+1}$ is independent of the auxiliary variable $c$. i.e. $s_{t+1}  \upmodels c\mid u$. 
\end{assumption}

Part of the above assumption is also used in \cite{lu2020sample}; considering the identifiability of iVAE, we also add some necessary assumptions, which are also practical. In the following discussion, we will show that when the underlying data-generating mechanism satisfies Assumption \ref{ass:data_gen}, the exogenous
variable $u$ can be identified up to permutation and affine transformations if the conditional prior
distribution $p(u | c)$ belongs to a general exponential family distribution.
\begin{assumption}\label{ass:dist}
The prior distribution of the exogenous variable $p(u| c)$ follows a general exponential family with its parameter specified by an arbitrary function $\boldsymbol{\lambda}  (c)$ and sufficient statistics $\boldsymbol{T}(u)=[\boldsymbol{T}_{\boldsymbol{d}}(u), \boldsymbol{T}_{NN}(u)]$, here $\boldsymbol{T}_{\boldsymbol{d}}(u)$ is defined by the concatenation of $\boldsymbol{T}_{\boldsymbol{d}}(u)=[\boldsymbol{T}_1(u^1)^T, \cdots, \boldsymbol{T}_d(u^d)^T]^T$ from a factorized exponential family and the outputs of a neural network $\boldsymbol{T}_{NN}(u)$ with universal approximation power. The probability density can be written as:
\begin{equation}
    p_{\boldsymbol{T},\boldsymbol{\lambda}}(u \mid  c) = \frac{\boldsymbol{Q}(u)}{\boldsymbol{Z}(u)}\exp \left[ \boldsymbol{T}(u)^T \boldsymbol{\lambda}(c)\right]. \label{eq:tlambda}
\end{equation}
\end{assumption}

Under the Assumption \ref{ass:dist}, and leveraging the ELBO in Equation \eqref{eq:elbo},  we can obtain the following identifiability of the parameters in the model. For convenience, we omit the subscript of $\boldsymbol{f}_{s_t, a_t}$ as $\boldsymbol{f}$.
\begin{theorem}
\label{thm:model_iden}
Assume that we observe data sampled from a generative model defined according to Equation \eqref{eq:model}-\eqref{eq:exo_noise} and Equation \eqref{eq:tlambda} with parameters $(\boldsymbol{f},\boldsymbol{T},\boldsymbol{\lambda})$, the following holds:
\begin{enumerate}[(i)]
    \item The set $\{s_{t+1}\in \mathcal{S}: \varphi_{\boldsymbol{\varepsilon}}(s_{t+1}=0)\}$ has measure zero, where $\varphi_{\boldsymbol{\varepsilon}}$ is the characteristic function of the density $p_{\boldsymbol{\varepsilon}}$ defined in Equation \eqref{eq:exo_noise}.
    \item The function $\boldsymbol{f}$ is injective and all of its second-order cross partial derivatives exist.
    \item The sufficient statistics $\boldsymbol{T}_{\boldsymbol{d}}$ are twice differentiable.
    \item There exist $k+1$ distinct points $c^0, \ldots, c^{k}$ such that the matrix
\begin{align}
 L=\left(\boldsymbol{\lambda}\left(c^1\right)-\boldsymbol{\lambda}\left(c^0\right), \ldots, \boldsymbol{\lambda}\left(c^{k}\right)-\boldsymbol{\lambda}\left(c^0\right)\right)   
\end{align}
of size $k \times k$ is invertible.
\end{enumerate}
Then, the parameters $\boldsymbol{\theta}=(\boldsymbol{f}, \boldsymbol{T}, \boldsymbol{\lambda})$ are identifiable up to an equivalence class induced by permutation and component-wise transformations.
\end{theorem}
Theorem \ref{thm:model_iden} guarantees the identifiability of Equation \eqref{eq:model}. We present its proof in Appendix \ref{app:proth1}. 

Moreover, Theorem \ref{thm:model_iden} further implies a consistent result on the conditional VAE. If the variational distribution of encoder $q_{\boldsymbol{\phi}}$ is a broad parametric family that includes the true posterior, we have the following results.

\begin{theorem}\label{thm:cons}
Assume the following holds:
\begin{enumerate}[(i)]
\item There exists the $({\boldsymbol{\theta}}, {\boldsymbol{\phi}})$ such that the family of distributions $q_{\boldsymbol{\phi}}\left(u \mid s_t, a_t, s_{t+1}, c\right)$ contains $p_{\boldsymbol{\theta}}\left(u \mid s_t, a_t, s_{t+1}, c\right)$.
\item We maximize $\mathcal{L}(\boldsymbol{\theta}, \boldsymbol{\phi})$ with respect to both $\boldsymbol{\theta}$ and $\boldsymbol{\phi}$.
\end{enumerate}
Then, given infinite data, OILCA can learn the true parameters $\boldsymbol{\theta}^*:=\left(\boldsymbol{f}^*, \boldsymbol{T}^*, \boldsymbol{\lambda}^*\right)$.
\end{theorem}

We present the corresponding proof in  Appendix \ref{app:proth2}. Theorem \ref{thm:cons} is proved by assuming our conditional VAE is flexible enough to ensure the ELBO is tight for some parameters and the optimization algorithm can achieve the global maximum of ELBO.

In our framework, the current generated expert sample pairs $(\tilde{s}_{t+1}, \tilde{a}_{t+1})$ are estimated based on the sampler policy $\hat{\pi}_E$. However, $\hat{\pi}_E$ may be not perfect, and its predicted results may contain noise. Thus, we would like to answer: ``given the noise level of the sampler policy, how many samples one need to achieve sufficiently well performance?''. Using $\kappa \in(0,0.5)$ indicates the noise level of $\hat{\pi}_E$. If $\hat{\pi}_E$ can exactly recover the true action $a_{t+1}$ (i.e., $\kappa=0$ ), then the generated sequences are perfect without any noise. On the contrary, $\kappa=0.5$ means that $\hat{\pi}_E$ can only produce random results, and the generated sequences are fully noisy. Then we have the following theorem:

\begin{theorem}\label{thm:influence}
Given a hypothesis class $\mathcal{H}$, for any $\epsilon, \delta \in(0,1)$ and $\kappa \in(0,0.5)$, if $\hat{\pi}_E \in \mathcal{H}$ is the pretrained policy model learned based on the empirical risk minimization (ERM), and the sample complexity (i.e., number of samples) is larger than $\frac{2 \log \left(\frac{2 \mathcal{H}}{\delta}\right)}{\epsilon^2(1-2 \kappa)^2}$, then the error between the model estimated and true results is smaller than $\epsilon$ with probability larger than $1-\delta$.
\end{theorem}

The related proof details are presented in Appendix \ref{app:proth3}. From Theorem \ref{thm:influence}, we can see: in order to guarantee the same performance with a given probability (i.e., $\epsilon$ and $\delta$ are fixed), one needs to generate more than $\frac{2 \log \left(\frac{2 \mathcal{H}}{\delta}\right)}{\epsilon^2(1-2 \kappa)^2}$ sequences. If the noise level $\kappa$ is larger, more samples have to be generated. Extremely, if the pre-trained policy can only produce fully noisy information (i.e., $\kappa=0.5$), then infinity number of samples are required, which is impossible in reality.

For the generalization ability, \cite{kaushik2020explaining} explains the efficacy of counterfactual augmented data by the empirical evidence. In the context of offline IL, a straightforward yet very relevant conclusion from the analysis of generalization ability is the strong dependence on the number of expert data \cite{ross2011reduction}. We work with finite state and action spaces $(|\mathcal{S}|,|\mathcal{A}|<\infty)$, and for the learned policy $\pi_{\boldsymbol{\omega}}$, we can analyze the generalization ability from the perspective of error upper bound with the optimal expert policy $\pi^*$.

\begin{theorem}\label{thm:generalization}
Let $|\mathcal{D}_E|$ be the number of empirical expert data used to train the policy and $\epsilon$ be the expected upper bound of generalization error. There exists constant $h$ such that, if
\begin{equation}
|\mathcal{D}_E|\geq \frac{h|\mathcal{S}|  \left|\mathcal{A}\right| \log (|\mathcal{S}| / \delta)}{{\epsilon}^2},
\end{equation}
and each state $s_t$ is sampled uniformly, then, with probability at least $1-\delta$, we have:
\begin{equation}
\max _{s_{t}}\left\|\pi^* (\cdot\mid s_{t})-\pi_{\boldsymbol{\omega}} (\cdot\mid s_{t}) \right\|_1 \leq {\epsilon}.
\end{equation}
which shows that increasing $|\mathcal{D}_E|$ drastically improves the generalization guarantee.
\end{theorem}
Note that we provide the proof details for Theorem \ref{thm:generalization} in Appendix ~\ref{app:proth4}.

\section{Experiments}
\label{sec:experiments}
In this section, we evaluate the performance of OILCA, aiming to answer the following questions. \textbf{Q1:} With the synthetic toy environment of the causal MDP, can OILCA disentangle the latent representations of the exogenous variable?  \textbf{Q2:} For an in-distribution test environment, can OILCA improve the performance? \textbf{Q3:} For an out-of-distribution test environment, can OILCA improve the generalization? In addition, we use five baseline methods: BC-exp (BC on the expert data $\mathcal{D}_E$), BC-all (BC on all demonstrations $\mathcal{D}_{\text{all}}$), ORIL \cite{zolna2020offline}, BCND \cite{sasaki2021behavioral}, LobsDICE \cite{kim2022lobsdice}, DWBC \cite{xu2022discriminator}. More details about these methods are presented in Appendix \ref{app:train_detail}. Furthermore, the hyper-parameters of our method and baselines are all detail-tuned for better performance.

\subsection{Simulations on Toy Environment (Q1)}
In real-world situations, we can hardly get the actual distributions of the exogenous variable.
To evaluate the disentanglement identifiability in variational counterfactual reasoning, we build a toy environment to show that our method can indeed get the disentangled distributions of the exogenous variable. 

\paragraph{Toy environment}For the toy environment, we consider a simple 2D navigation task. The agent can move a fixed distance in each of the four directions. States are continuous and are considered as the 2D position of the agent. The goal is to navigate to a specific target state within a bounding box. The reward is the negative distance between the agent's state and the target state. For the initialization of the environment, we consider $C$ kinds of Gaussian distributions (Equation \eqref{eq:tlambda}) for the exogenous variable, where $C=3$, and the conditional variable $c$ is the class label. We design the transition function by using a multi-layer perceptron (MLP) with invertible activations. We present the details about data collection and model learning in Appendix \ref{app:train_detail}. 
\begin{figure*}[!t]
\begin{minipage}[!t]{0.43\textwidth} 
\begin{figure}[H]
    \centering
    \subfigure[]{\includegraphics[width=0.3\linewidth]{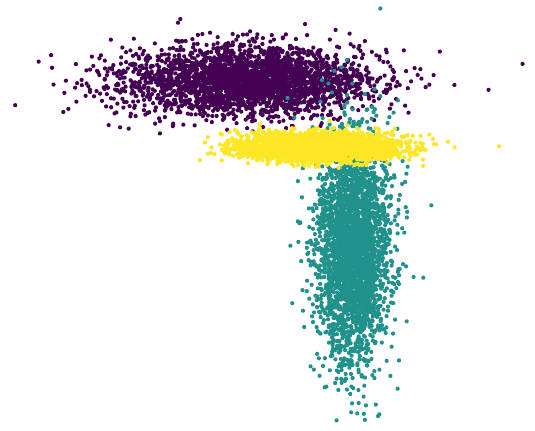}}~
    \subfigure[]{\includegraphics[width=0.3\linewidth]{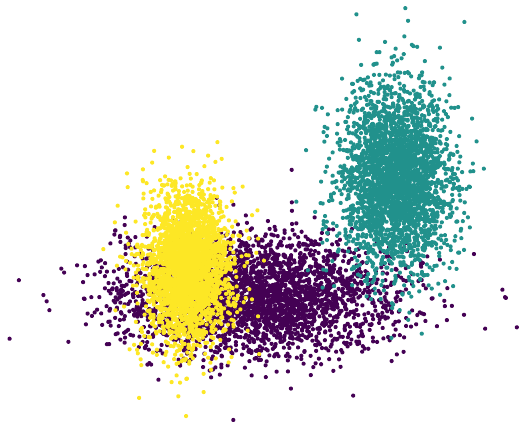}}~
    \subfigure[]{\includegraphics[width=0.3\linewidth]{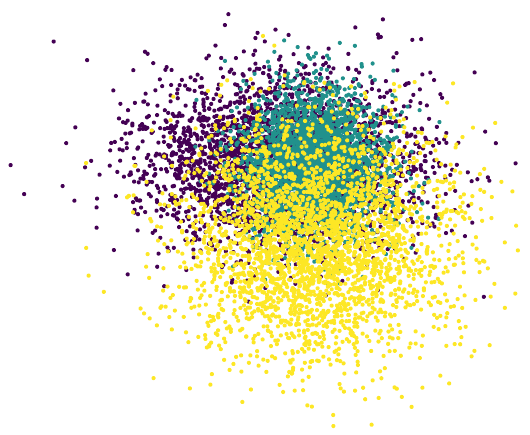}}
    \caption{Visualization of both observation and latent spaces of the exogenous variable. (a) Samples from the true distribution of the sources $p_{\boldsymbol{\theta}^*}(u | c)$. (b) Samples from the
posterior $q_{\boldsymbol{\phi}}\left(u | s_t, a_t, s_{t+1}, c\right)$. (c) Samples from the
posterior $q_{\boldsymbol{\phi'}}\left(u| s_t, a_t, s_{t+1}\right)$ without class label.}
    \label{fig:exo}  
\end{figure}
 \end{minipage}\quad
\begin{minipage}[!t]{0.25\textwidth} 
\begin{figure}[H]
    \centering
    \includegraphics[width=1\linewidth]{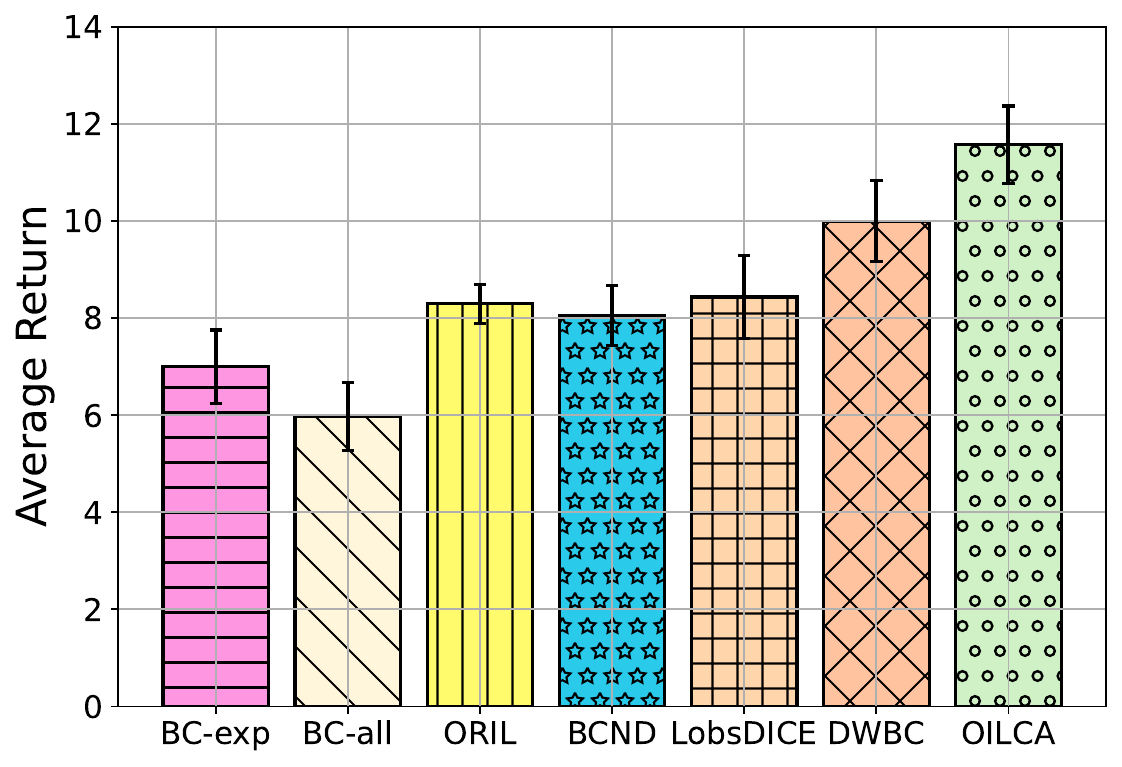}
    \caption{Performance of OILCA and baselines in the toy environment. We plot the mean and the standard errors of the averaged return over five random seeds.}
    \label{fig:simu_result}
    \end{figure}
\end{minipage} \quad
\begin{minipage}[!h]{0.25\textwidth} 
\begin{figure}[H]
    \centering
    \includegraphics[width=1\linewidth]{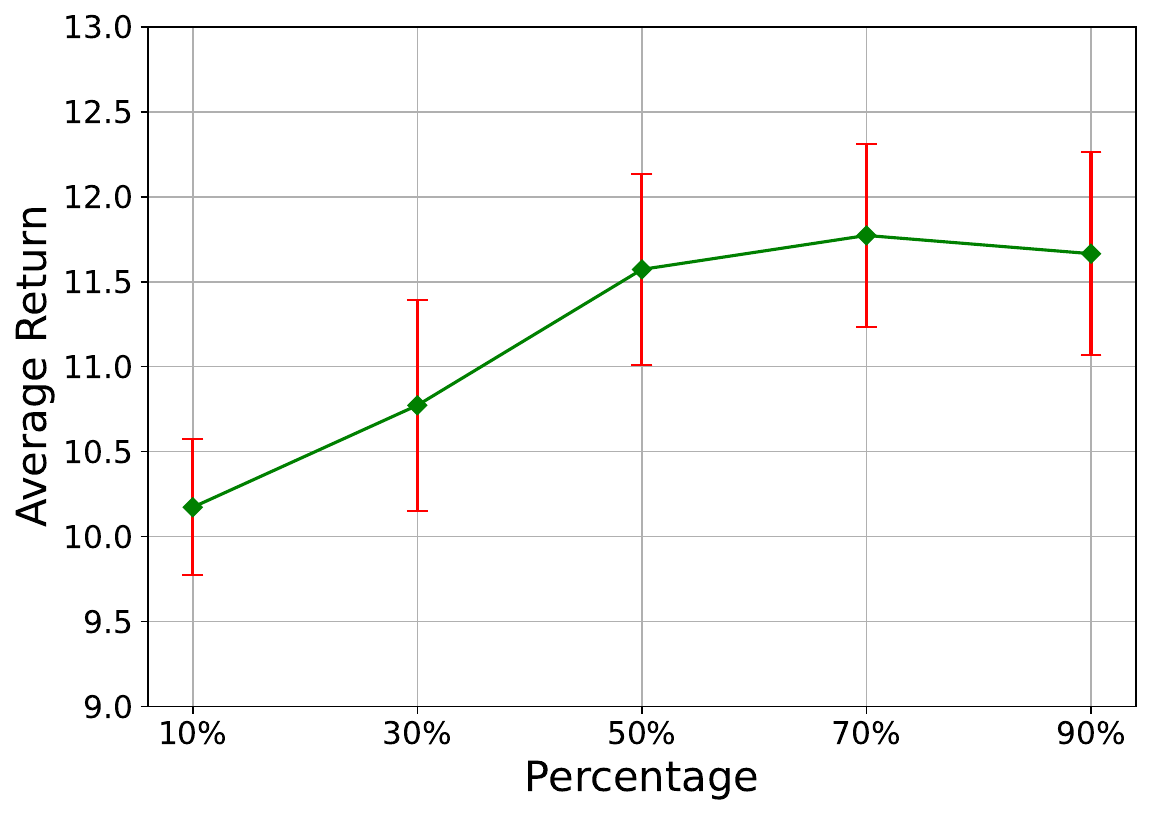}
    \caption{Performance of OILCA with the growing percentage of $|\mathcal{D}_{E}| /|\mathcal{D}_{U}|$. We plot the average return's mean and standard errors over five random seeds.}
    \label{fig:ablation_sim}
    \end{figure}
\end{minipage}
\end{figure*}

\paragraph{Results} We visualize the identifiability in such a 2D case in Figure \ref{fig:exo}, where we plot the sources, the posterior distributions learned by OILCA, and the posterior distributions learned by a vanilla conditional VAE, respectively. Our method recovers the original sources to trivial indeterminacies (rotation and sign flip), whereas the vanilla conditional VAE fails to separate the latent variables well. To show the effectiveness of our method in the constructed toy environment, we also conduct repeated experiments for 1K episodes (500 steps per episode) to compare OILCA with all baseline methods. The results are presented in Figure \ref{fig:simu_result}, showing that OILCA achieves the highest average return. 
To analyze the influence of the number of augmented samples (Theorem \ref{thm:generalization}), we also conduct the experiments with varying numbers of counterfactual expert data; the result is shown in Figure \ref{fig:ablation_sim}. The X-axis represents the percentage of $|\mathcal{D}_{E}| /|\mathcal{D}_{U}|$, where $\mathcal{D}_{E}$ is the augmented expert data. We can observe that within a certain interval, the generalization ability of the learned policy does improve with the increase of $|\mathcal{D}_{E}|$.

\subsection{In-distribution Performance (Q2)}

We use \textsc{DeepMind Control Suite} \cite{tassa2018deepmind} to evaluate the performance of in-distribution performance. Similar to \cite{zolna2020offline}, we also define an episode as positive if its episodic reward is among the top 20\% episodes for the task. For the auxiliary variable $c$, we add three kinds of different Gaussian noise distributions into the environment (encoded as $c=\{0,1,2\}$). Moreover, we present the detailed data collection and statistics in Appendix \ref{app:data_generate}).  
We report the results in Table \ref{tab:dmc}. We can observe that OILCA outperforms baselines on all tasks; this shows the effectiveness of the augmented counterfactual expert data. We also find that the performance of ORIL and DWBC tends to decrease in testing for some tasks (ORIL: \textsc{Cheetah Run}, \textsc{Walker Walk}; DWBC: \textsc{Cheetah Run}, \textsc{Manipulator Insert Ball}); this ``overfitting'' phenomenon also occurs in experiments of previous offline RL works \cite{kumar2019stabilizing,wu2019behavior}. This is perhaps due to limited data size and model generalization bottleneck. DWBC gets better results on most tasks, but for some tasks, such as Manipulator Insert Ball and the Walker Walk, OILCA achieves more than twice the average return than DWBC.

\begin{table}[!t]
\renewcommand\arraystretch{1.4}
\caption{Results for in-distribution performance on \textsc{DeepMind Control Suite}. We report the average return of episodes (with the length of 1K steps) over five random seeds. The best results and second best results are \textbf{bold} and \underline{underlined}, respectively.}\label{tab:dmc}
\resizebox{\linewidth}{!}{
\begin{tabular}{l|ccccccc}
\toprule
 \texttt{Task Name}               & BC-exp & BC-all &ORIL  &BCND  & LobsDICE & DWBC & \textbf{OILCA} \\ \midrule
\textsc{Cartpole Swingup}        &   \makecell[c]{195.44\\ \scriptsize{$\pm$ 7.39}}     &   \makecell[c]{269.03\\ \scriptsize{$\pm$ 7.06}}     &   \makecell[c]{221.24\\ \scriptsize{$\pm$ 14.49}}   &   \makecell[c]{243.52\\ \scriptsize{$\pm$ 11.33}}   &  \makecell[c]{292.96\\ \scriptsize{$\pm$ 11.05}}        &  \makecell[c]{\underline{382.55}\\ \scriptsize{$\pm$ 8.95}}    &   \makecell[c]{\textbf{608.38}\\ \scriptsize{$\pm$ 35.54}}   \\ 
                                                                                      \textsc{Cheetah Run}             &   \makecell[c]{66.59\\ \scriptsize{$\pm$ 11.09}}     &   \makecell[c]{90.01\\ \scriptsize{$\pm$ 31.74}}     &  \makecell[c]{45.08 \\\scriptsize{$\pm$ 9.88}}    &  \makecell[c]{\underline{96.06}\\\scriptsize{$\pm$ 16.15}}    &   \makecell[c]{74.53\\\scriptsize{$\pm$ 7.75}}        &  \makecell[c]{66.87\\\scriptsize{$\pm$  4.60}}   &   \makecell[c]{\textbf{116.05}\\\scriptsize{$\pm$ 14.65}}   \\ 
                                                                                      \textsc{Finger Turn Hard}        &  \makecell[c]{129.20\\ \scriptsize{$\pm$ 4.51}}      &  \makecell[c]{104.56\\ \scriptsize{$\pm$ 8.32}}      &  \makecell[c]{185.57\\ \scriptsize{$\pm$ 26.75}}    &  \makecell[c]{204.67\\ \scriptsize{$\pm$ 13.18}}    &  \makecell[c]{190.93\\ \scriptsize{$\pm$ 12.19}}         &  \makecell[c]{\underline{243.47}\\ \scriptsize{$\pm$ 17.12}}     &  \makecell[c]{\textbf{298.73}\\ \scriptsize{$\pm$ 5.11}}    \\ 
                                                                                      \textsc{Fish Swim}               &   \makecell[c]{74.59\\\scriptsize{$\pm$11.73}}     &   \makecell[c]{68.87\\\scriptsize{$\pm$ 11.93}}     & \makecell[c]{84.90\\\scriptsize{$\pm$1.96}}     &  \makecell[c]{153.28\\\scriptsize{$\pm$19.29}}    &    \makecell[c]{188.84\\\scriptsize{$\pm$11.28}}      &   \makecell[c]{\underline{212.39}\\\scriptsize{$\pm$7.62}}    & \makecell[c]{\textbf{290.28}\\\scriptsize{$\pm$ 10.07}}     \\ 
                                                                                      \textsc{Humanoid Run}            &    \makecell[c]{77.59\\\scriptsize{$\pm$8.63}}    &   \makecell[c]{138.93\\\scriptsize{$\pm$9.14}}     &   \makecell[c]{96.88\\\scriptsize{$\pm$10.76}}  &   \makecell[c]{257.01\\\scriptsize{$\pm$11.21}}    &   \makecell[c]{120.87\\\scriptsize{$\pm$10.66}}       &    \makecell[c]{\underline{302.33}\\\scriptsize{$\pm$14.53}}  &   \makecell[c]{\textbf{460.96}\\\scriptsize{$\pm$17.55}}   \\ 
                                                                                      \textsc{Manipulator Insert Ball} &   \makecell[c]{91.61\\\scriptsize{$\pm$7.49}}     &    \makecell[c]{98.59\\\scriptsize{$\pm$1.62}}    &  \makecell[c]{98.25\\\scriptsize{$\pm$2.68}}    &  \makecell[c]{141.71\\\scriptsize{$\pm$15.30}}    &   \makecell[c]{\underline{197.79}\\\scriptsize{$\pm$1.98}}       &   \makecell[c]{107.86\\\scriptsize{$\pm$15.01}}   &    \makecell[c]{\textbf{296.83}\\\scriptsize{$\pm$3.43}}  \\ 
                                                                                      \textsc{Manipulator Insert Peg}  &   \makecell[c]{92.02\\\scriptsize{$\pm$15.04}}     &  \makecell[c]{119.63\\\scriptsize{$\pm$6.37}}       &  \makecell[c]{105.86\\\scriptsize{$\pm$5.10}}    &   \makecell[c]{220.66\\\scriptsize{$\pm$15.14}}   &   \makecell[c]{\underline{299.19}\\\scriptsize{$\pm$3.40}}       &  \makecell[c]{238.39\\\scriptsize{$\pm$19.76}}    &   \makecell[c]{\textbf{305.66}\\\scriptsize{$\pm$4.91}}   \\ 
                                                                                      \textsc{Walker Stand}            &   \makecell[c]{169.14\\\scriptsize{$\pm$ 8.00}}     & \makecell[c]{192.14\\\scriptsize{$\pm$ 37.91}}      &  \makecell[c]{181.23\\\scriptsize{$\pm$ 10.31}}    &  \makecell[c]{279.66\\\scriptsize{$\pm$ 12.69}}    &   \makecell[c]{252.34\\\scriptsize{$\pm$ 7.73} }      &  \makecell[c]{\underline{280.07}\\\scriptsize{$\pm$ 5.79}}    &  \makecell[c]{\textbf{395.51}\\\scriptsize{$\pm$ 8.05}}   \\ 
                                                                                      \textsc{Walker Walk}             &  \makecell[c]{29.44\\\scriptsize{$\pm$ 2.18}}      &  \makecell[c]{75.79\\\scriptsize{$\pm$ 6.34}}      &  \makecell[c]{41.43\\\scriptsize{$\pm$ 4.05}}    &   \makecell[c]{157.44\\\scriptsize{$\pm$ 9.31}}   &  \makecell[c]{102.14\\\scriptsize{$\pm$ 5.94}}        &   \makecell[c]{\underline{166.95}\\\scriptsize{$\pm$ 10.68}}   &     \makecell[c]{\textbf{377.19}\\\scriptsize{$\pm$ 15.87}}  \\ \bottomrule
\end{tabular}}
\end{table}

\subsection{Out-of-distribution Generalization (Q3)} 
To evaluate the out-of-distribution generalization of OILCA, we use a benchmark named \textsc{CausalWorld} \cite{ahmed2020causalworld}, which provides a combinatorial family of tasks with a common causal structure and underlying factors. Furthermore, the auxiliary variable $c$ represents the different \textit{do}-interventions on the task environments \cite{ahmed2020causalworld}. We collect the offline data by using three different \textit{do}-interventions on environment features (\texttt{stage\_color}, \texttt{stage\_friction}. \texttt{floor\_friction}) to generate offline datasets, while other features are set as the defaults. More data collection and statistics details are presented in Appendix \ref{app:data_generate}. 
We show the comparative results in Table \ref{tab:cw}. It is evident that OILCA also achieves the best performance on all the tasks; ORIL and DWBC perform the second-best results on most tasks. BCND performs poorly compared to other methods. The reason may be that the collected data in space $\mathbf{A}$ can be regarded as low-quality data when evaluating on space $\mathbf{B}$, which may lead to the poor performance of a single BC and cause cumulative errors for the BC ensembles of BCND. LobsDICE even performs poorer than BC-exp on most tasks. This is because the KL regularization, which regularizes the learned policy to stay close to $\mathcal{D}_{U}$ is too conservative, resulting in a suboptimal policy, especially when $\mathcal{D}_{U}$ contains a large collection of noisy data. This indeed hurts the performance of LobsDICE for out-of-distribution generalization.

\begin{table}[!t]
\renewcommand\arraystretch{1.4}
\caption{Results for out-of-distribution generalization on \textsc{CausalWorld}. We report the average return of episodes (length varies for different tasks) over five random seeds. All the models are trained on \textit{space} $\mathbf{A}$ and tested on \textit{space} $\mathbf{B}$ to show the out-of-distribution performance \cite{ahmed2020causalworld}. }\label{tab:cw}
\resizebox{\linewidth}{!}{
\begin{tabular}{l|ccccccc}
\toprule
 \texttt{Task Name}               & BC-exp & BC-all &ORIL  &BCND  & LobsDICE & DWBC & \textbf{OILCA} \\ \midrule
                                                                                      \textsc{Reaching}                 &    \makecell[c]{281.18\\\scriptsize{$\pm$ 16.45}}    &    \makecell[c]{176.54\\\scriptsize{$\pm$ 9.75}}    &  \makecell[c]{339.40\\\scriptsize{$\pm$ 12.98}}    &   \makecell[c]{228.33\\\scriptsize{$\pm$ 7.14}}   &   \makecell[c]{243.29\\\scriptsize{$\pm$ 9.84}}       &  \makecell[c]{\underline{479.92}\\\scriptsize{$\pm$ 18.75}}    &   \makecell[c]{\textbf{976.60}\\\scriptsize{$\pm$ 20.13}}   \\ 
                                                                                      \textsc{Pushing}                 &   \makecell[c]{256.64\\\scriptsize{$\pm$ 12.70}}     &  \makecell[c]{235.58\\\scriptsize{$\pm$ 10.23}}      &  \makecell[c]{283.91\\\scriptsize{$\pm$ 19.72}}    &  \makecell[c]{191.23\\\scriptsize{$\pm$ 12.64}}    &    \makecell[c]{206.44\\\scriptsize{$\pm$ 15.35}}      &  \makecell[c]{\underline{298.09}\\\scriptsize{$\pm$ 14.94}}    &  \makecell[c]{\textbf{405.08}\\\scriptsize{$\pm$ 24.03}}    \\ 
                                                                                      \textsc{Picking}          &  \makecell[c]{270.01\\\scriptsize{$\pm$ 13.13}}      &   \makecell[c]{258.54\\\scriptsize{$\pm$ 16.53}}     &   \makecell[c]{\underline{388.15}\\\scriptsize{$\pm$ 19.21}}   &  \makecell[c]{221.89\\\scriptsize{$\pm$ 7.68}}    &     \makecell[c]{337.78\\\scriptsize{$\pm$ 12.09}}     &     \makecell[c]{366.26\\\scriptsize{$\pm$ 8.77}} &    \makecell[c]{\textbf{491.09}\\\scriptsize{$\pm$ 6.44}}  \\ 
                                                                                      \textsc{Pick and Place}               & \makecell[c]{294.06\\\scriptsize{$\pm$ 7.34}}       &  \makecell[c]{225.42\\\scriptsize{$\pm$ 12.44}}      &  \makecell[c]{270.75\\\scriptsize{$\pm$ 14.87}}    &  \makecell[c]{259.12\\\scriptsize{$\pm$ 8.01}}    &  \makecell[c]{266.09\\\scriptsize{$\pm$ 10.31}}        & \makecell[c]{\underline{349.66}\\\scriptsize{$\pm$ 7.39}}     &   \makecell[c]{\textbf{490.24}\\\scriptsize{$\pm$ 11.69}}   \\ 
                                                                                      \textsc{Stacking2}          &  \makecell[c]{\underline{496.63}\\\scriptsize{$\pm$ 7.68}}       &  \makecell[c]{394.91\\\scriptsize{$\pm$ 16.98}}      &   \makecell[c]{388.55\\\scriptsize{$\pm$ 10.93}}   &  \makecell[c]{339.18\\\scriptsize{$\pm$ 9.46}}    &   \makecell[c]{362.47\\\scriptsize{$\pm$ 17.05}}       &   \makecell[c]{481.07\\\scriptsize{$\pm$ 10.11}}   &  \makecell[c]{\textbf{831.82}\\\scriptsize{$\pm$ 11.78}}     \\ 
                                                                                      \textsc{Towers}                 & \makecell[c]{667.81\\\scriptsize{$\pm$ 9.27}}       & \makecell[c]{\underline{784.88}\\\scriptsize{$\pm$ 17.17}}       & \makecell[c]{655.96\\\scriptsize{$\pm$ 15.14}}     &  \makecell[c]{139.30\\\scriptsize{$\pm$ 18.22}}    & \makecell[c]{535.74\\\scriptsize{$\pm$ 13.76}}         &   \makecell[c]{768.68\\\scriptsize{$\pm$ 24.77}}   & \makecell[c]{\textbf{994.82}\\\scriptsize{$\pm$ 5.76}}     \\  
                                                                                      \textsc{Stacked Blocks} &   \makecell[c]{581.91\\\scriptsize{$\pm$ 26.92}}     &      \makecell[c]{452.88\\\scriptsize{$\pm$ 18.78}}   &     \makecell[c]{702.15\\\scriptsize{$\pm$ 15.30}}  &   \makecell[c]{341.97\\\scriptsize{$\pm$ 33.69}}    &        \makecell[c]{250.44\\\scriptsize{$\pm$ 14.08}}   &      \makecell[c]{\underline{1596.96}\\\scriptsize{$\pm$ 81.84}} &   \makecell[c]{\textbf{2617.71}\\\scriptsize{$\pm$ 88.07}}    \\ 
                                                                                      \textsc{Creative Stacked Blocks} &  \makecell[c]{496.32\\\scriptsize{$\pm$ 26.92}}      &   \makecell[c]{529.82\\\scriptsize{$\pm$ 31.01}}     & \makecell[c]{\underline{882.27}\\\scriptsize{$\pm$ 46.79}}     &  \makecell[c]{288.55\\\scriptsize{$\pm$ 19.63}}    &         \makecell[c]{317.95\\\scriptsize{$\pm$ 32.03}} &   \makecell[c]{700.23\\\scriptsize{$\pm$ 13.71}}   &  \makecell[c]{\textbf{1348.49}\\\scriptsize{$\pm$ 55.05}}    \\ 
                                                                                      \textsc{General}                  &\makecell[c]{492.78\\\scriptsize{$\pm$17.64}}        &   \makecell[c]{547.32\\\scriptsize{$\pm$8.49}}     &  \makecell[c]{\underline{647.95}\\\scriptsize{$\pm$24.39}}    &   \makecell[c]{195.06\\\scriptsize{$\pm$9.80}}   &   \makecell[c]{458.27\\\scriptsize{$\pm$18.69}}       &   \makecell[c]{585.98\\\scriptsize{$\pm$19.25}}   &   \makecell[c]{\textbf{891.14}\\\scriptsize{$\pm$23.12}}   \\ \bottomrule
\end{tabular}}
\end{table}

\section{Conclusion}

In this paper, we propose an effective and generalizable offline imitation learning framework OILCA, which can learn a policy from the expert and unlabeled demonstrations. We apply a novel identifiable counterfactual expert data augmentation approach to facilitate the offline imitation learning. We also analyze the influence of generated data and the generalization ability theoretically. The experimental results demonstrate that our method achieves better performance in simulated and public tasks.

\vspace{-0.2cm}
\section*{Acknowledgement}

This work is supported in part by National Key R\&D Program of China (2022ZD0120103), National Natural Science Foundation of China (No. 62102420), Beijing Outstanding Young Scientist Program NO. BJJWZYJH012019100020098, Intelligent Social Governance Platform, Major Innovation \& Planning Interdisciplinary Platform for the ``Double-First Class'' Initiative, Renmin University of China, Public Computing Cloud, Renmin University of China, fund for building world-class universities (disciplines) of Renmin University of China, Intelligent Social Governance Platform.

\bibliographystyle{plainnat}
\bibliography{refer.bib}
\newpage
\setcounter{section}{0}
\appendix

\section{Proof of Theorem \ref{thm:model_iden}}\label{app:proth1}

In this section, we provide proof for the disentanglement identifiability of the inferred exogenous variable. Our proof consists of three main components. It is worth noting that we also use $\boldsymbol{f}$ to replace $\boldsymbol{f}_{s_t, a_t}$ for simplicity.
\begin{proof}
The following are the three steps:

\textbf{Step I} We use the first assumption in Theorem \ref{thm:model_iden} to demonstrate that the observed data distributions are equivalent to the noiseless distributions. Specifically, suppose that we have two sets of parameters $(\boldsymbol{f}, \boldsymbol{T}, \boldsymbol{\lambda})$ and $(\tilde{\boldsymbol{f}}, \tilde{\boldsymbol{T}}, \tilde{\boldsymbol{\lambda}})$, such that for all pairs $(s_{t+1}, c)$ ($(s,c)$ for simplicity), we have:
\begin{align}
   \tilde{p}_{\boldsymbol{T}, \boldsymbol{\lambda}, \boldsymbol{f}, c}(s)=\tilde{p}_{\tilde{\boldsymbol{T}}, \tilde{\boldsymbol{f}}, \tilde{\boldsymbol{\lambda}}, c}(s) 
\end{align}

\begin{align}
p_{\boldsymbol{\theta}}(s \mid c)&=p_{\tilde{\boldsymbol{\theta}}}(s \mid c)\\
\Longrightarrow \int p_{\boldsymbol{\varepsilon}}(s-\boldsymbol{f}(u)) p_{\boldsymbol{T}, \boldsymbol{\lambda}}(u \mid c) d u&=\int p_{\boldsymbol{\varepsilon}}(s-\tilde{\boldsymbol{f}}(u)) p_{\tilde{\boldsymbol{T}}, \tilde{\boldsymbol{\lambda}}}(u \mid c) d u
\end{align}
\begin{align}
\resizebox{\linewidth}{!}{$
\Longrightarrow \int p_{\boldsymbol{\varepsilon}}(s-\overline{s}) p_{\boldsymbol{T}, \boldsymbol{\lambda}}\left(\boldsymbol{f}^{-1}(\overline{s} \mid c) \operatorname{vol}\left(J_{f^{-1}}(\overline{s})\right) d \overline{s} \right.=\int p_{\boldsymbol{\varepsilon}}(s-\overline{s}) p_{\boldsymbol{T}, \boldsymbol{\lambda}}\left(\tilde{\boldsymbol{f}}^{-1}(\overline{s} \mid c) \operatorname{vol}\left(J_{\tilde{f}^{-1}}(\overline{s})\right) d \overline{s}\right.$}\label{eq:jacco}
\end{align}
\begin{align}
\Longrightarrow \int p_{\boldsymbol{\varepsilon}}(s-\overline{s}) \tilde{p}_{\boldsymbol{T}, \boldsymbol{\lambda}, \boldsymbol{f}, c}(\overline{s}) d \overline{s}&= \int p_{\boldsymbol{\varepsilon}}(s-\overline{s}) \tilde{p}_{\tilde{\boldsymbol{T}}, \tilde{\boldsymbol{f}}, \tilde{\boldsymbol{\lambda}}, c}(\overline{s}) d \overline{s}\label{eq:trans}\\
\Longrightarrow\left(\tilde{p}_{\boldsymbol{T}, \boldsymbol{\lambda}, \boldsymbol{f}, c} * p_{\boldsymbol{\varepsilon}}\right)(s)&=\left(\tilde{p}_{\tilde{\boldsymbol{T}}, \tilde{\boldsymbol{f}}, \tilde{\boldsymbol{\lambda}}, c} * p_{\boldsymbol{\varepsilon}}\right)(s)\label{eq:conv}\\
\Longrightarrow F\left[\tilde{p}_{\boldsymbol{T}, \boldsymbol{\lambda}, \boldsymbol{f}, c}\right](\boldsymbol{\omega}) \varphi_{\boldsymbol{\varepsilon}}(\boldsymbol{\omega})&=F\left[\tilde{p}_{\tilde{\boldsymbol{T}}, \tilde{\boldsymbol{f}}, \tilde{\boldsymbol{\lambda}}}\right](\boldsymbol{\omega}) \varphi_{\boldsymbol{\varepsilon}}(\boldsymbol{\omega})\label{eq:four}\\
\Longrightarrow F\left[\tilde{p}_{\boldsymbol{T}, \boldsymbol{\lambda}, \boldsymbol{f}, c}\right](\boldsymbol{\omega}) &=F\left[\tilde{p}_{\tilde{\boldsymbol{T}}, \tilde{\boldsymbol{f}}, \tilde{\boldsymbol{\lambda}}, c}\right](\boldsymbol{\omega})\label{eq:drop} \\
\Longrightarrow \tilde{p}_{\boldsymbol{T}, \boldsymbol{\lambda}, \boldsymbol{f}, c}(s) &=\tilde{p}_{\tilde{\boldsymbol{T}}, \tilde{\boldsymbol{f}}, \boldsymbol{\boldsymbol { \lambda }}, c}(s).\label{eq:last}
\end{align}

where:
\begin{itemize}
\item in Equation \eqref{eq:jacco}, $J$ denotes the Jacobian, and we make the change of variable $\overline{s}=\boldsymbol{f}(u)$ on the left-hand side, and $\overline{s}=\tilde{\boldsymbol{f}}(u)$ on the right-hand side.
\item in Equation \eqref{eq:trans}, we introduce
\begin{align}
\tilde{p}_{\boldsymbol{T}, \boldsymbol{\lambda}, \boldsymbol{f}, c} \triangleq p_{\boldsymbol{T}, \boldsymbol{\lambda}}\left(\boldsymbol{f}^{-1}\right)(s \mid c) \operatorname{vol}\left(J_{f^{-1}}(s)\right) \mathbb{I}(s)\label{eq:dingyi}
\end{align}

\item  in Equation \eqref{eq:conv}, $*$ denotes the convolution operator.
\item  in Equation \eqref{eq:four}, $F$ denotes the Fourier transformation and $\varphi_{\boldsymbol{\varepsilon}}=F\left[p_{\boldsymbol{\varepsilon}}\right]$.
\item  in Equation \eqref{eq:drop}, $\varphi_{\boldsymbol{\varepsilon}}(\boldsymbol{w})$ is dropped because it is non-zero almost everywhere according to the first assumption of Theorem \ref{thm:model_iden}.
\end{itemize}

Equation \eqref{eq:last} is valid for all $(s, c)$. What is basically says is that for the distributions to be the same after adding the noise, the noise-free distributions have to be the same. Note that $s$ here is a general variable, and we are actually dealing with the noise-free probability densities.

\textbf{Step II} Using Equation \eqref{eq:dingyi} to substitute Equation \eqref{eq:last}, we have
\begin{align}
p_{\boldsymbol{T}, \boldsymbol{\lambda}}\left(\boldsymbol{f}^{-1}\right)(s \mid c) \operatorname{vol}\left(J_{f^{-1}}(s)\right) \mathbb{I}(s)=p_{\tilde{\boldsymbol{T}}}, \tilde{\boldsymbol{\lambda}}\left(\tilde{\boldsymbol{f}}^{-1}\right)(s \mid c) \operatorname{vol}\left(J_{\tilde{f}^{-1}}(s)\right) \mathbb{I}(s) .
\end{align}
Then, we can apply logarithm on the above equation and substitute $p_{\boldsymbol{T}, \boldsymbol{\lambda}}$ with its definition in Equation \eqref{eq:exo_noise}, and obtain
\begin{align}
\log \operatorname{vol}\left(J_{f^{-1}}(s)\right) \log Q\left(\boldsymbol{f}^{-1} s\right) & -\log Z(c)+\left\langle\boldsymbol{T}\left(\boldsymbol{f}^{-1}(s)\right), \boldsymbol{\lambda}(c)\right\rangle\nonumber \\
& =\log \operatorname{vol}\left(J_{\tilde{f}^{-1}}(s)\right) \log \tilde{Q}\left(\tilde{\boldsymbol{f}}^{-1} s\right)-\log \tilde{Z}(c)+\left\langle\tilde{\boldsymbol{T}}\left(\tilde{\boldsymbol{f}}^{-1}(s)\right), \tilde{\boldsymbol{\lambda}}(c)\right\rangle
\end{align}
Let $c^0, \cdots, c^k$ be the $k+1$ points defined in the fourth assumption of Theorem \ref{thm:model_iden}, we can obtain $k+1$ equation. By subtracting the first equation from the remaining $k$ equations, we then obtain:
\begin{align}
\left\langle\boldsymbol{T}\left(\boldsymbol{f}^{-1}(s)\right), \boldsymbol{\lambda}\left(c^l\right)-\boldsymbol{\lambda}\left(c^0\right)\right\rangle& +\log \frac{Z\left(c^0\right)}{Z\left(c^l\right)}\nonumber \\
& =\left\langle\tilde{\boldsymbol{T}}\left(\tilde{\boldsymbol{f}}^{-1}(s)\right), \tilde{\boldsymbol{\lambda}}\left(c^l\right)-\tilde{\boldsymbol{\lambda}}\left(c^0\right)\right\rangle+\log \frac{\tilde{Z}\left(c^0\right)}{\tilde{Z}\left(c^l\right)},
\end{align}
where $l=1, \cdots, k$. Let $\boldsymbol{b} \in \mathbb{R}^k$ in which $b_l=\log \frac{\tilde{Z}\left(c^0\right) Z\left(c^l\right)}{\tilde{Z}\left(c^l\right) Z\left(c^0\right)}$, we have
\begin{align}
L^T \boldsymbol{T}\left(\boldsymbol{f}^{-1}(s)\right)=\tilde{L} \tilde{\boldsymbol{T}}\left(\tilde{\boldsymbol{f}}^{-1}(s)\right)+\boldsymbol{m}
\end{align}
Finally, we multiply both side by $L^{-T}$ and obtain
\begin{align}
\boldsymbol{T}\left(\boldsymbol{f}^{-1}(s)\right)=A \tilde{\boldsymbol{T}}\left(\tilde{\boldsymbol{f}}^{-1}(s)\right)+\boldsymbol{n}. \label{eq:step2}
\end{align}
where $A=L^{-T} L$ and $\boldsymbol{n}=L^{-T} \boldsymbol{m}$.

\textbf{Step III} Now recall the definition of $\boldsymbol{T}$ and the third assumption. We start by evaluating Equation \eqref{eq:step2} at $k+1$ points of $u^l, s^l$ and obtain $k+1$ equations. Then, we subtract the first equation from the remaining $k+1$ equations:
\begin{align}
&{\left[\boldsymbol{T}\left(u_1\right)-\boldsymbol{T}\left(u^0\right), \cdots, \boldsymbol{T}\left(u^k\right)-\boldsymbol{T}\left(u^0\right)\right]} \nonumber\\
=&A\left[\tilde{\boldsymbol{T}}\left(\tilde{\boldsymbol{f}}^{-1}\left(s^1\right)\right)-\tilde{\boldsymbol{T}}\left(\tilde{\boldsymbol{f}}^{-1}\left(s^0\right)\right), \cdots, \tilde{\boldsymbol{T}}\left(\tilde{\boldsymbol{f}}^{-1}\left(s^l\right)\right)-\tilde{\boldsymbol{T}}\left(\tilde{\boldsymbol{f}}^{-1}\left(s^0\right)\right)\right] .
\end{align}
Next, we only need to show that for $u_0$, there exist $k$ points $u^1, \cdots, u^k$ such that the columns are linear independent, which can be proven by contradiction. Suppose that there exists no such $u^l\in\{ u^0,\cdots, u^k\}$, then $\left\langle\boldsymbol{T}(u^l)-\boldsymbol{T}\left(u^0\right), \boldsymbol{\lambda}\right\rangle=0$ and thus $\boldsymbol{T}(u^l)=\boldsymbol{T}\left(u^0\right)= \mathrm{const}$. This contradicts with the assumption that the prior distribution is strongly exponential. Therefore, there must exist $k+1$ points such that the transformation is invertible. Then we have $(\boldsymbol{f}, \boldsymbol{T}, \boldsymbol{\lambda}) \sim(\tilde{\boldsymbol{f}}, \tilde{\boldsymbol{T}}, \tilde{\boldsymbol{\lambda}})$.
\end{proof}

\section{Proof of Theorem \ref{thm:cons}}\label{app:proth2}
According to Equation \eqref{eq:elbo}, 
if the family $q_{\boldsymbol{\phi}}\left(u \mid s_t, a_t, s_{t+1}, c\right)$ is large enough to include $p_{\boldsymbol{\theta}}\left(u \mid s_{t+1}, s_t, a_t, c\right)$, then by optimizing the loss over its parameter $\boldsymbol{\phi}$, we will minimize the KL term, eventually reaching zero, and the loss will be equal to the log-likelihood. 

The conditional VAE, in this case, inherits all the properties of maximum likelihood estimation. In this particular case, since our identifiability is guaranteed up to equivalence classes, the consistency of MLE means that we converge to the equivalence class (Theorem \ref{thm:model_iden}) of true parameter $\boldsymbol{\theta}^*$ \textit{i.e}. Under the condition of infinite data.
\section{Proof of Theorem \ref{thm:influence}}\label{app:proth3}
Suppose the prediction error of $\hat{\pi}_E$ is $e$ (\textit{i.e}., $\sum \mathbb{I}(a_t^{\hat{\pi}_E} \neq a_t )=e), a_t$ is the true action that an expert take, then the mismatching probability between the observed and predicted results comes from two parts: (1) The observed result is true, but the prediction is wrong, that is, $e(1-\kappa)$. (2) The observed result is wrong, but the prediction is right, that is $(1-e) \kappa$. Thus, the total mismatching probability is $\kappa+e(1-2 \kappa)$.

The following proof is based on the reduction to absurdity. We first propose an assumption and then derive contradicts to invalidate the assumption.

\textbf{Assumption.} Suppose the prediction error of $\hat{\pi}_E$ (i.e., $e$ ) is larger than $\epsilon$. Then, at least one of the following statements hold: 
\begin{enumerate}[(1)]
    \item The empirical mis-matching rate of $\hat{\pi}_E$ is smaller than $\kappa+\frac{\epsilon(1-2 \kappa)}{2}$. 
    \item The empirical mis-matching rate of the optimal $h^* \in \mathcal{H}$ (\textit{i.e.}, the prediction error of $h^*$ is 0$)$ is larger than $\kappa+\frac{\epsilon(1-2 \kappa)}{2}$. 
\end{enumerate}

These statements are easy to understand, since if both of them do not hold, we can conclude that the empirical loss of $\hat{\pi}_E$ is larger than that of $h^*$, which does not agree with the ERM definition.

\textbf{Contradicts.} To begin with, we review the uniform convergence properties~\cite{shalev2014understanding} by the following lemma:

\begin{lemma}\label{lemma:pro3}
Let $\mathcal{H}$ be a hypothesis class, then for any $\epsilon \in(0,1)$ and $h \in \mathcal{H}$, if the number of training samples is $m$, the following formula holds:
$$
\mathbb{P}(|R(h)-\hat{R}(h)|>\epsilon)<2|\mathcal{H}| \exp \left(-2 m \epsilon^2\right)
$$
where $R$ and $\hat{R}$ are the expectation and empirical losses, respectively.
\end{lemma}

For statement (1), since the prediction error of $\hat{\pi}_E$ is larger than $\epsilon$, the expectation loss $R(\hat{\pi}_E)$ is larger than $\kappa+\epsilon(1-2 \kappa)$. If the empirical loss $\hat{R}(\hat{\pi}_E)$ is smaller than $\kappa+\frac{\epsilon(1-2 \kappa)}{2}$, then $|R(\hat{\pi}_E)-\hat{R}(\hat{\pi}_E)|$ should be larger than $\frac{\epsilon(1-2 \kappa)}{2}$. At the same time, according to Lemma \ref{lemma:pro3}, when the sample number $m$ is larger than $\frac{2 \log \left(\frac{2|\mathcal{H}|}{\delta}\right)}{\epsilon^2(1-2 \kappa)^2}$, we have $\mathbb{P}\left(|R(\hat{\pi}_E)-\hat{R}(\hat{\pi}_E)|>\frac{\epsilon(1-2 \kappa)}{2}\right)<\delta$.

For statement (2), the expectation loss of $h^*$ is $\kappa$, i.e., $R\left(h^*\right)=\kappa$. If the empirical loss $\hat{R}\left(h^*\right)$ is larger than $\kappa+\frac{\epsilon(1-2 \kappa)}{2}$, then $\mid R\left(h^*\right)-$ $\hat{R}\left(h^*\right) \mid$ should be larger than $\frac{\epsilon(1-2 \kappa)}{2}$. According to Lemma \ref{lemma:pro3}, when the sample number $m$ is larger than $\frac{2 \log \left(\frac{2|\mathcal{H}|}{\delta}\right)}{\epsilon^2(1-2 \kappa)^2}$, we have $\mathbb{P}\left(\left|R\left(h^*\right)-\hat{R}\left(h^*\right)\right|>\frac{\epsilon(1-2 \kappa)}{2}\right)<\delta$.

As a result, both of the above statements hold with the probability smaller than $\delta$, which implies that the prediction error of $\hat{\pi}_E$ is smaller than $\epsilon$ with the probability larger than $1-\delta$.

\section{Proof of Theorem \ref{thm:generalization}}\label{app:proth4}
\begin{lemma}\label{lemma}
(Proposition A.8 of Agarwal et al. \cite{alek2019reinforce}). Let $z$ be a discrete random variable that takes values in $\{1, \ldots, d\}$, distributed according to $q$. We write $q$ as a vector where $\vec{q}=[\operatorname{Pr}(z=j)]_{j=1}^d$. Assume we have $n$ i.i.d. samples, and that our empirical estimate of $\vec{q}$ is $[\vec{q}]_j=\sum_{i=1}^n \mathbf{1}\left[z_i=j\right] / n$. We have that $\forall {\epsilon}>0$ :
$$
\operatorname{Pr}\left(\|\hat{q}-\vec{q}\|_2 \geq 1 / \sqrt{n}+{\epsilon}\right) \leq e^{-n {\epsilon}^2}
$$
which implies that:
$$
\operatorname{Pr}\left(\|\hat{q}-\vec{q}\|_1 \geq \sqrt{d}(1 / \sqrt{n}+{\epsilon})\right) \leq e^{-n {\epsilon}^2}
$$
\end{lemma}

\begin{proof}
 Applying Lemma \ref{lemma}, we have that for considering a fixed $s_t$, wp. at least $1-\delta$:
\begin{equation}
 \|\pi(\cdot \mid s_t)-\pi_{\boldsymbol{\omega}}(\cdot \mid s_t)\|_1 \leq h \sqrt{\frac{|\mathcal{A}| \log (1 / \delta)}{n}} 
\end{equation}
where $n$ is the number of expert data used to estimate $\pi_{\boldsymbol{\omega}}(\cdot \mid s_t)$. Then we apply the union bound across all states and actions to get that wp. at least $1-\delta$:
\begin{equation}
\max_{s_t} \|\pi(\cdot \mid s_t)-\pi_{\boldsymbol{\omega}}(\cdot \mid s_t)\|_1 \leq h \sqrt{\frac{|\mathcal{S}||\mathcal{A}| \log (|\mathcal{S}| / \delta)}{n}} 
 \end{equation}
The result follows by rearranging $n$ and relabeling $h$.   
\end{proof}

\begin{remark}
How much counterfactual expert data can we generate using our OILCA framework? Supposing we have $n$ independent state action tuples in the expert data, we run the data augmentation module for $m$ times, which means that we can augment each state to $m$ counterfactual states and subsequently to $m$ corresponding counterfactual actions. Thus, in total, we can obtain $n^m$ counterfactual tuples--an exponential increase for the previously given expert data. Back to Theorem~\ref{thm:generalization}, this demonstrates that our OILCA can effectively enhance the policy's generalization ability.
\end{remark}
\section{Training Details}\label{app:train_detail}

\subsection{Data Generation and Statistics}\label{app:data_generate}

\paragraph{Toy Environment} The dimensions of state and action are both 2. For the exogenous variable, we generate the non-stationary 2D Gaussian data as follows: $u^* \mid c \sim \mathcal{N}\left({\mu}(c), {\operatorname { d i a g }}\left({\sigma}^2(c)\right)\right)$, where $c$ is the class label. $\mu_1(c)=0$ for all $c$ and $\mu_2(c)=\alpha \gamma(c)$, where $\alpha \in \mathbb{R}$ and $\gamma$ is a permutation. The variance $\sigma^2(c)$ is generated randomly and independently across the classes. For the transition function, we use an MLP to generate the next state $s_{t+1}$, such that $s_{t+1}=\mathrm{MLP}(s_t, a_t, u_{t+1})$, where $u_{t+1}$ is the sample of $u$ at timestep $t+1$. For each class of exogenous variables, we generate 1K episodes for the data collection (500 steps per episode). Similar to \textsc{DeepMind Control Suite}, we also define a positive episode if its reward is among the top 20\% episodes, and each of these positives is randomly chosen to constitute $\mathcal{D}_E$ with $\frac{1}{10}$ chance. As a result, we choose 75 episodes in $\mathcal{D}_E$ and 925 episodes in $\mathcal{D}_U$. For the online testing, we can evaluate all the methods on the toy environment with any kind of distribution of the exogenous variable. 

\paragraph{\textsc{DeepMind Control Suite}} \textsc{DeepMind Control Suite} (Figure \ref{fig:dmc}) contains a variety of continuous control tasks involving locomotion and simple manipulation. States consist of joint angles and velocities, and action spaces vary depending on the task. The episodes are 1000 steps long, and the environment reward is continuous, with a maximum value of 1 per step. 
During the collection of offline data, we apply random Gaussian perturbation to the action outputted by the policy. This perturbation is specified in the XML configuration file as an integral part of the environment. Additionally, the distribution of the perturbation differs across different environment initialization (auxiliary variable $c$) due to their initialization seeds. In particular, different seeds correspond to different mean and variance of the Gaussian distribution perturbation via the random number generator. This approach is employed to introduce uncertainty into the environment~\cite{}, thereby aligning with our problem setting.
We define an episode as positive if its episodic return is among the top 20\% episodes; each of these positives is randomly chosen to constitute $\mathcal{D}_E$ with $\frac{1}{10}$ chance. We present the details in Table \ref{tab:data_stat}.

\begin{minipage}{0.48\textwidth}
\begin{table}[H]
    \centering
    \begin{tabular}{lrr}
    \toprule
    \textbf{Task} & \textbf{Total} & $\boldsymbol{\mathcal{D}_E}$ \\
    \midrule
    Cartpole swingup         &   40 &   2 \\
    Cheetah run              &  300 &   3 \\
    Finger turn hard         &  500 &   9 \\
    Fish swim                &  200 &   1 \\
    Humanoid run             & 3000 &  53 \\
    Manipulator insert ball  & 1500 &  30 \\
    Manipulator insert peg   & 1500 &  23 \\
    Walker stand             &  200 &   4 \\
    Walker walk              &  200 &   6 \\
    \midrule
    Reaching        &   600 &   12 \\
    Pushing              &  600 &   13 \\
    Picking         &  600 &   15 \\
    Pick and Place                &  600 &   12 \\
    Stacking2             & 600 &  11 \\
    Towers  & 600 &  13 \\
    Stacked Blocks   & 600 &  13 \\
    Creative Stacked Block             &  600 &   14 \\
    General              &  600 &   12 \\
    \bottomrule
    \end{tabular}
    \vspace{0.3cm}
   \caption{\textbf{Datasets statistics.} The total number of episodes and corresponding number of expert demonstrations ($\mathcal{D}_E$) per task.}\label{tab:data_stat}
    \label{tab:episodes}
\end{table}
\end{minipage}\quad
\begin{minipage}{0.48\textwidth}
\begin{figure}[H]
    \centering
    \includegraphics[width=1\linewidth]{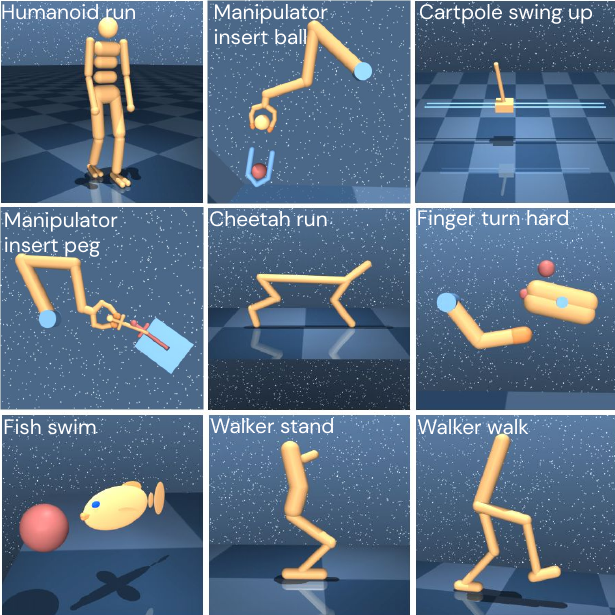}
    \caption{\textsc{DeepMind Control Suite} is a set of popular continuous control environments with tasks of varying difficulties, including locomotion and simple object manipulation.}\label{fig:dmc}
\end{figure}  
\end{minipage}

\paragraph{\textsc{CausalWorld}} \textsc{CausalWorld} provides a combinatorial family of such tasks with common causal structure and underlying factors (including, e.g., robot and object masses, colors, sizes) (Figure \ref{fig:causal_world}). We conduct the offline dataset collection process by using various online behavior policies. We collect the mixed dataset by using three kinds of do-interventions (Figure \ref{fig:do_causal}) on different environment features. And we divide the offline dataset into $\mathcal{D}_E$ and $\mathcal{D}_U$, similar to the \textsc{DeepMind Control Suite}. The detailed statistics about the dataset are presented in Table \ref{tab:data_stat}.

\subsection{Detailed Descriptions of Baselines}\label{app:baseline}
\begin{enumerate}[$\bullet$]
    \item \textbf{BC-exp:} Behavioral cloning on expert data $\mathcal{D}_E$. $\mathcal{D}_E$ owns higher quality data but fewer quantities and thus causes serious compounding error problems to the resulting policy.
    \item \textbf{BC-all:} Behavioral cloning on all data $\mathcal{D}_{\text{all}}$. BC-all can generalize better than BC-exp due to access to a much larger dataset, but its performance may be negatively impacted by the low-quality data in $\mathcal{D}_{\text{all}}$.
    \item \textbf{ORIL}~\cite{zolna2020offline}\textbf{:} ORIL learns a reward function and uses it to solve
an offline RL problem. It suffers from high computational costs and the difficulty of performing offline RL under distributional shifts.
    \item \textbf{BCND}~\cite{sasaki2021behavioral}\textbf{:} BCND is trained on all data, and it reuses another
policy learned by BC as the weight of the original BC objective. Its performance will be worse if the suboptimal data occupies the major part of the offline dataset.
    \item \textbf{LobsDICE}~\cite{kim2022lobsdice}\textbf{:} LobsDICE optimizes in the space of state-action stationary distributions and state-transition stationary distributions rather than in the space of policies.
    \item \textbf{DWBC}~\cite{xu2022discriminator}\textbf{:} DWBC is trained on all data. It mainly designs a new IL algorithm, where the discriminator outputs serve as the weights of the BC loss.
\end{enumerate}

\begin{figure}[!t]
    \centering
    \subfigure[Pushing]{\includegraphics[width=0.245\textwidth]{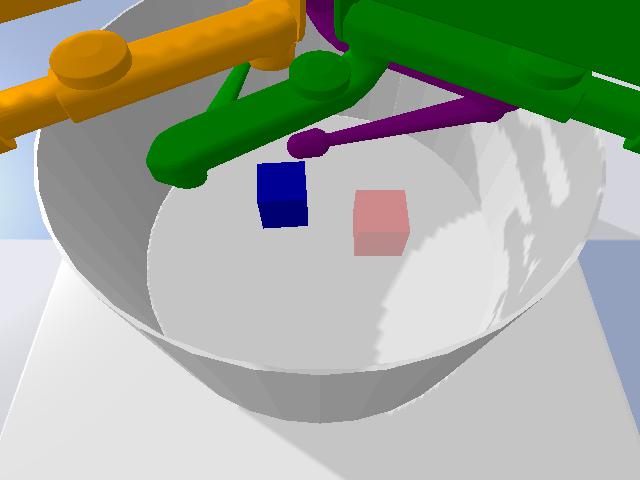}}
    \subfigure[Picking]{\includegraphics[width=0.245\textwidth]{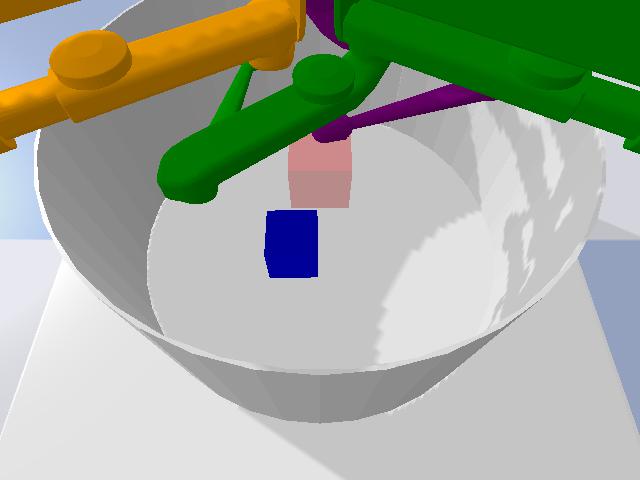}}
    \subfigure[Pick and Place]{\includegraphics[width=0.245\textwidth]{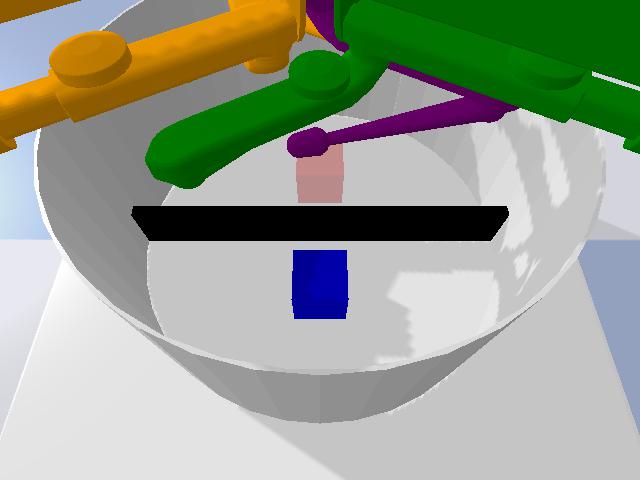}}
    \subfigure[Stacking2]{\includegraphics[width=0.245\textwidth]{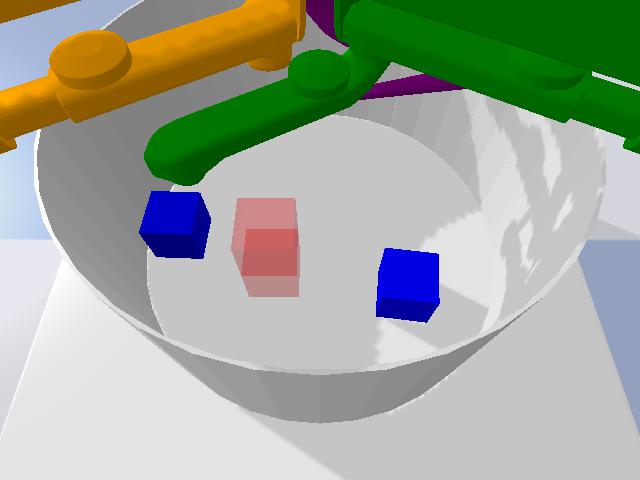}}\\
    \subfigure[Stacked Blocks]{\includegraphics[width=0.245\textwidth]{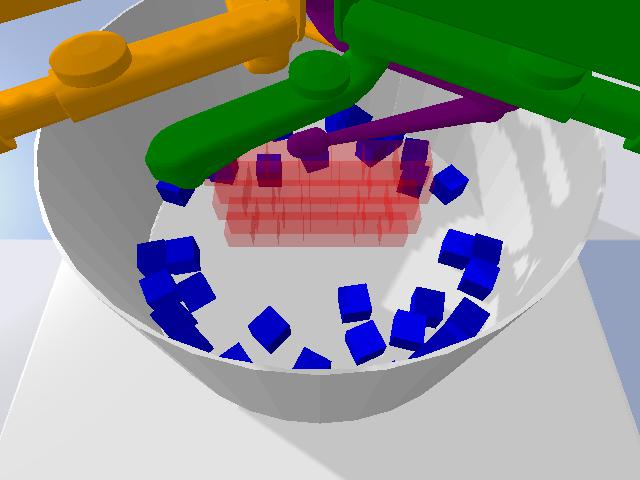}}
    \subfigure[General]{\includegraphics[width=0.245\textwidth]{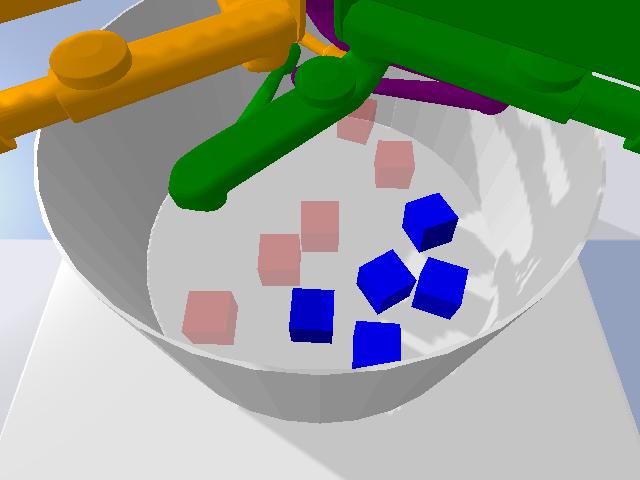}}
    \subfigure[CreativeStackedBlocks]{\includegraphics[width=0.245\textwidth]{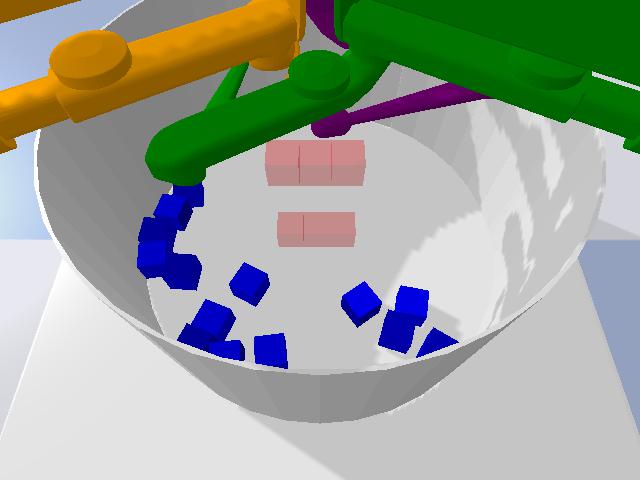}}
    \subfigure[Towers]{\includegraphics[width=0.245\textwidth]{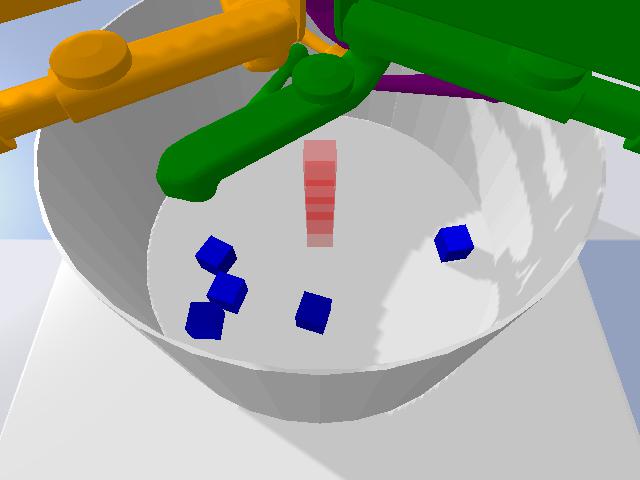}}
    \caption{Example tasks from the task generators provided in the \textsc{Causalworld}. The goal shape is visualized in opaque red, and the blocks are visualized in blue.}
    \label{fig:causal_world}
\end{figure}   

\begin{figure}[h]
    \centering
    \includegraphics[width=0.5\linewidth]{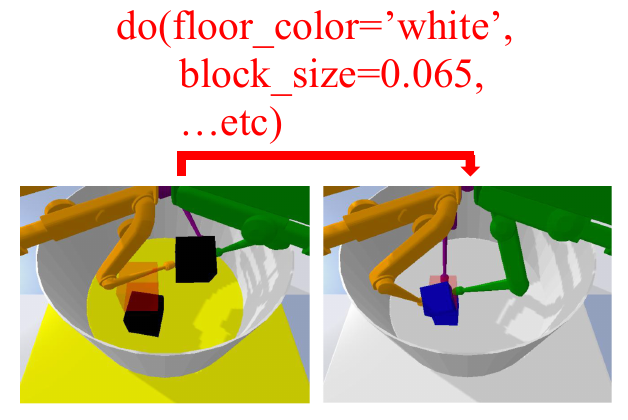}
    \caption{Example of \textit{do}-interventions on exposed variables in \textsc{CausalWorld}.}\label{fig:do_causal}
\end{figure}

\section{Additional Results}\label{app:result}

\subsection{In-distribution Experiments on \textsc{CausalWorld}}

To further show the in-distribution performance, we supplement the experiments on \textsc{CausalWorld}, in which both training and testing are conducted on space $\mathbf{A}$. The results are shown in Table \ref{tab:cwaa}. In most tasks, our OILCA still achieves the highest average episode return, demonstrating our method's effectiveness across different scenarios. Especially comparing the results in Table~\ref{tab:cw} and Table~\ref{tab:cwaa}, we can notice that the advantage of OILCA for out-of-distribution generalization is more obvious. This proves the strong generalization ability of the counterfactual data augmentation module, which makes the offline imitation learning policy more robust to the data distribution shift. This point is especially significant in out-of-distribution scenarios, where the data distribution shifts more intensely.

\begin{table}[!t]
\renewcommand\arraystretch{1.4}
\caption{Results for in-distribution performance on \textsc{CausalWorld}. We report the average return of episodes (length varies for different tasks) over five random seeds. All the models are trained on \textit{space} $\mathbf{A}$ and tested on \textit{space} $\mathbf{A}$ to show the in-distribution performance \cite{ahmed2020causalworld}. The best results and second best results are \textbf{bold} and \underline{underlined}, respectively. }\label{tab:cwaa}
\resizebox{\linewidth}{!}{
\begin{tabular}{l|ccccccc}
\toprule
 \texttt{Task Name}               & BC-exp & BC-all &ORIL  &BCND  & LobsDICE & DWBC & \textbf{OILCA} \\ \midrule
                                                                                      \textsc{Reaching}                 &    \makecell[c]{353.98\\\scriptsize{$\pm$ 11.48}}    &    \makecell[c]{247.60\\\scriptsize{$\pm$ 15.99}}    &  \makecell[c]{372.39\\\scriptsize{$\pm$ 9.58}}    &   \makecell[c]{358.36\\\scriptsize{$\pm$ 15.45}}   &   \makecell[c]{323.43\\\scriptsize{$\pm$ 10.13}}       &  \makecell[c]{\underline{530.96}\\\scriptsize{$\pm$ 8.70}}    &   \makecell[c]{\textbf{986.19}\\\scriptsize{$\pm$10.27}}   \\ 
                                                                                      \textsc{Pushing}                 &   \makecell[c]{ 331.32\\\scriptsize{$\pm$ 6.36}}     &  \makecell[c]{ 310.62\\\scriptsize{$\pm$ 9.21}}      &  \makecell[c]{364.37\\\scriptsize{$\pm$ 8.36}}    &  \makecell[c]{335.87\\\scriptsize{$\pm$ 9.02}}    &    \makecell[c]{275.38\\\scriptsize{$\pm$ 9.93}}      &  \makecell[c]{\underline{436.22} \\\scriptsize{$\pm$ 5.39}}    &  \makecell[c]{\textbf{579.55}\\\scriptsize{$\pm$ 12.64}}    \\ 
                                                                                      \textsc{Picking}          &  \makecell[c]{394.63\\\scriptsize{$\pm$ 12.98}}      &   \makecell[c]{360.28\\\scriptsize{$\pm$ 8.98}}     &   \makecell[c]{427.39\\\scriptsize{$\pm$ 13.69}}   &  \makecell[c]{381.45\\\scriptsize{$\pm$ 8.63}}    &     \makecell[c]{326.97\\\scriptsize{$\pm$ 12.31}}     &     \makecell[c]{\underline{479.05}\\\scriptsize{$\pm$ 8.57}} &    \makecell[c]{\textbf{648.34}\\\scriptsize{$\pm$ 8.51}}  \\ 
                                                                                      \textsc{Pick and Place}               & \makecell[c]{\underline{453.59}\\\scriptsize{$\pm$ 7.58}}       &  \makecell[c]{355.83\\\scriptsize{$\pm$ 8.47}}      &  \makecell[c]{348.34\\\scriptsize{$\pm$ 11.63}}    &  \makecell[c]{376.34\\\scriptsize{$\pm$ 9.87}}    &  \makecell[c]{ 287.81\\\scriptsize{$\pm$ 10.06}}        & \makecell[c]{448.89\\\scriptsize{$\pm$12.49 }}     &   \makecell[c]{\textbf{588.87}\\\scriptsize{$\pm$ 9.29}}   \\ 
                                                                                      \textsc{Stacking2}          &  \makecell[c]{596.14\\\scriptsize{$\pm$ 15.76}}       &  \makecell[c]{435.12\\\scriptsize{$\pm$ 12.81}}      &   \makecell[c]{467.11\\\scriptsize{$\pm$ 13.19}}   &  \makecell[c]{476.33\\\scriptsize{$\pm$ 5.21}}    &   \makecell[c]{378.3\\\scriptsize{$\pm$ 7.65}}       &   \makecell[c]{\underline{631.75}\\\scriptsize{$\pm$ 8.54}}   &  \makecell[c]{\textbf{920.18}\\\scriptsize{$\pm$ 7.36}}     \\ 
                                                                                      \textsc{Towers}                 & \makecell[c]{723.49\\\scriptsize{$\pm$ 15.82}}       & \makecell[c]{\underline{947.96}\\\scriptsize{$\pm$ 17.56}}       & \makecell[c]{679.93\\\scriptsize{$\pm$ 8.68}}     &  \makecell[c]{680.61\\\scriptsize{$\pm$ 8.57}}    & \makecell[c]{735.79\\\scriptsize{$\pm$ 12.23}}         &   \makecell[c]{915.26\\\scriptsize{$\pm$ 17.97}}   & \makecell[c]{\textbf{1263.94}\\\scriptsize{$\pm$ 8.98}}     \\  
                                                                                      \textsc{Stacked Blocks} &   \makecell[c]{1320.97\\\scriptsize{$\pm$ 19.83}}     &      \makecell[c]{947.96\\\scriptsize{$\pm$ 25.45}}   &     \makecell[c]{1520.62\\\scriptsize{$\pm$ 31.62}}  &   \makecell[c]{1247.96\\\scriptsize{$\pm$ 29.14}}    &        \makecell[c]{958.64\\\scriptsize{$\pm$ 26.56}}   &      \makecell[c]{\underline{2116.51}\\\scriptsize{$\pm$ 32.97}} &   \makecell[c]{\textbf{3210.23}\\\scriptsize{$\pm$43.63 }}    \\ 
                                                                                      \textsc{Creative Stacked Blocks} &  \makecell[c]{684.52\\\scriptsize{$\pm$ 16.69}}      &   \makecell[c]{593.41\\\scriptsize{$\pm$ 26.86}}     & \makecell[c]{758.04\\\scriptsize{$\pm$ 12.70}}     &  \makecell[c]{\underline{933.88}\\\scriptsize{$\pm$ 16.57}}    &         \makecell[c]{601.18\\\scriptsize{$\pm$ 19.42}} &   \makecell[c]{870.29\\\scriptsize{$\pm$ 24.56}}   &  \makecell[c]{\textbf{1476.41}\\\scriptsize{$\pm$ 25.94}}    \\ 
                                                                                      \textsc{General}                  &\makecell[c]{626.15\\\scriptsize{$\pm$20.57}}        &   \makecell[c]{691.37\\\scriptsize{$\pm$17.22}}     &  \makecell[c]{\textbf{1072.05}\\\scriptsize{$\pm$47.26}}    &   \makecell[c]{572.70\\\scriptsize{$\pm$11.28}}   &   \makecell[c]{549.89\\\scriptsize{$\pm$15.31}}       &   \makecell[c]{786.44\\\scriptsize{$\pm$18.52}}   &   \makecell[c]{\underline{964.32}\\\scriptsize{$\pm$ 17.08}}   \\ \bottomrule
\end{tabular}}
\end{table}

\subsection{Combinations with Other Base Offline IL Methods}
To validate that the effectiveness of our method is not restricted by the base offline IL methods, we combine the \underline{C}ounterfactual data \underline{A}ugmentation (CA) part with ORIL and BCND, which are represented as ORIL+CA and BCND+CA, respectively. Also, we conduct corresponding experiments on the benchmarks in this paper, and the results are shown in Table~\ref{tab:dmcw}. From the table, we can observe that the CA module can always help improve policy performance regardless of the base policy choice, which demonstrates its wide applicability. Besides, referring to the results in Table~\ref{tab:dmcw}, Table~\ref{tab:dmc}, and Table~\ref{tab:cw}, we can find that ORIL+CA and BCND+CA outperform all methods without CA's assistance in most tasks, which implies that the simple counterfactual data augmentation may even work better than the complicated learning method designs.


\begin{table}[!t]
\renewcommand\arraystretch{1.4}
\caption{Results for in-distribution performance on part of tasks in \textsc{Deepmind Control Suite} and out-of-distribution generalization on part of tasks in \textsc{CausalWorld}. We report the average return of episodes (length varies for different tasks) over five random seeds. The training and testing procedures follow those introduced in Section~\ref{sec:experiments}. All the results obtained by CA-assisted methods are \textbf{bold} to highlight the effect of the counterfactual data augmentation module.}\label{tab:dmcw}
\resizebox{\linewidth}{!}{
\begin{tabular}{cl|cccccc}
\toprule
\multicolumn{1}{c|}{\texttt{Benchmark}}& \multicolumn{1}{c|}{\texttt{Task Name}}                & ORIL & ORIL+CA & BCND  & BCND+CA  & DWBC & OILCA \\ \midrule
                                                                                     \multicolumn{1}{l|}{\multirow{4}{*}{\begin{tabular}[c]{@{}c@{}} \\\textsc{Deepmind}\\ \textsc{Control} \\ \textsc{Suite}\end{tabular}}} & \textsc{Cartpole Swingup}                 &    \makecell[c]{221.24\\\scriptsize{$\pm$ 14.49}}    &    \makecell[c]{\textbf{426.79}\\\scriptsize{$\pm$ 12.09}}    &  \makecell[c]{243.52\\\scriptsize{$\pm$ 11.33}}    &   \makecell[c]{\textbf{452.68}\\\scriptsize{$\pm$ 12.86}}   &   \makecell[c]{382.55\\\scriptsize{$\pm$ 8.95}}       &  \makecell[c]{\textbf{608.38}\\\scriptsize{$\pm$ 35.54}}       \\ \multicolumn{1}{l|}{} &
                                                                                      \textsc{Cheetah Run}                 &   \makecell[c]{45.08\\\scriptsize{$\pm$9.88 }}     &  \makecell[c]{\textbf{78.44}\\\scriptsize{$\pm$ 6.95}}      &  \makecell[c]{96.06\\\scriptsize{$\pm$ 16.15}}    &  \makecell[c]{\textbf{158.62}\\\scriptsize{$\pm$ 8.85}}    &    \makecell[c]{66.87\\\scriptsize{$\pm$ 4.60}}      &  \makecell[c]{\textbf{116.05}\\\scriptsize{$\pm$ 14.65}}      \\ \multicolumn{1}{l|}{} &
                                                                                      \textsc{Finger Turn Hard}          &  \makecell[c]{185.57\\\scriptsize{$\pm$ 26.75}}      &   \makecell[c]{\textbf{227.94}\\\scriptsize{$\pm$ 15.47}}     &   \makecell[c]{204.67\\\scriptsize{$\pm$ 13.18}}   &  \makecell[c]{\textbf{284.29}\\\scriptsize{$\pm$ 12.03}}    &     \makecell[c]{243.47\\\scriptsize{$\pm$ 17.12}}     &     \makecell[c]{\textbf{298.73}\\\scriptsize{$\pm$ 5.11}}   \\ \multicolumn{1}{l|}{} &
                                                                                      \textsc{Fish Swim}               & \makecell[c]{84.90\\\scriptsize{$\pm$ 1.96}}       &  \makecell[c]{\textbf{156.92}\\\scriptsize{$\pm$ 8.18}}      &  \makecell[c]{153.28\\\scriptsize{$\pm$ 19.29}}    &  \makecell[c]{\textbf{268.56}\\\scriptsize{$\pm$ 6.03}}    &  \makecell[c]{212.39\\\scriptsize{$\pm$ 7.62}}        & \makecell[c]{\textbf{290.29}\\\scriptsize{$\pm$ 10.07}}       \\ \hline 
                                                                                      \multicolumn{1}{l|}{\multirow{4}{*}{\begin{tabular}[c]{@{}c@{}} \\ \textsc{Causal} \\ \textsc{World} \\ \end{tabular}}} &
                                                                                      \textsc{Reaching}          &  \makecell[c]{339.40\\\scriptsize{$\pm$ 12.98}}       &  \makecell[c]{\textbf{652.21}\\\scriptsize{$\pm$ 7.05}}      &   \makecell[c]{228.33\\\scriptsize{$\pm$ 7.14}}   &  \makecell[c]{\textbf{582.44}\\\scriptsize{$\pm$9.07 }}    &   \makecell[c]{479.92\\\scriptsize{$\pm$ 18.75}}       &   \makecell[c]{\textbf{976.60}\\\scriptsize{$\pm$ 20.13}}      \\ \multicolumn{1}{l|}{} &
                                                                                      \textsc{Pushing}                 & \makecell[c]{283.91\\\scriptsize{$\pm$ 19.72}}       & \makecell[c]{\textbf{367.46} \\\scriptsize{$\pm$ 6.31}}       & \makecell[c]{191.23\\\scriptsize{$\pm$ 12.64}}     &  \makecell[c]{\textbf{320.94}\\\scriptsize{$\pm$ 10.37}}    & \makecell[c]{298.09\\\scriptsize{$\pm$ 14.94}}         &   \makecell[c]{\textbf{405.08}\\\scriptsize{$\pm$ 24.03}}       \\  \multicolumn{1}{l|}{} &
                                                                                      \textsc{Picking} &   \makecell[c]{388.15\\\scriptsize{$\pm$ 19.21}}     &      \makecell[c]{\textbf{458.03}\\\scriptsize{$\pm$13.95 }}   &     \makecell[c]{221.89\\\scriptsize{$\pm$ 7.68}}  &   \makecell[c]{\textbf{486.32}\\\scriptsize{$\pm$8.03 }}    &        \makecell[c]{366.26\\\scriptsize{$\pm$ 8.77}}   &      \makecell[c]{\textbf{491.09}\\\scriptsize{$\pm$ 6.44}}    \\ \multicolumn{1}{l|}{} &
                                                                                      \textsc{Pick and Place} &  \makecell[c]{270.75\\\scriptsize{$\pm$ 14.87}}      &   \makecell[c]{\textbf{372.18}\\\scriptsize{$\pm$ 10.74}}     & \makecell[c]{259.12\\\scriptsize{$\pm$ 8.01}}     &  \makecell[c]{\textbf{393.59}\\\scriptsize{$\pm$7.81 }}    &         \makecell[c]{349.66\\\scriptsize{$\pm$ 7.39}} &   \makecell[c]{\textbf{490.24}\\\scriptsize{$\pm$ 11.69}}     \\ 
                                                                                   \bottomrule
\end{tabular}}
\end{table}

\subsection{Performance of changing the auxiliary variable $c$}
To show the influence of the different choices of the auxiliary variable $c$, we conduct additional experiments on the  \textsc{CausalWorld} benchmark. Specially, for the change of $c$'s choice, we apply the similar $do$-interventions to more features (\textit{i.e.} block color, block mass) and fewer features. The performance of our OILCA under different intervened features (different choices of $c$) is shown in Table \ref{tab:c_change}. Specially, $C=1$ represents feature set \texttt{stage\_friction}, $C=2$ represents feature set (\texttt{stage\_friction}, \texttt{floor\_friction}), $C=3$ represents feature set (\texttt{\small stage\_color}, \texttt{\small stage\_friction}, \texttt{\small floor\_friction}), $C=4$ represents feature set (\texttt{\small stage\_color}, \texttt{\small stage\_friction}, \texttt{\small floor\_friction}, \texttt{\small block\_mass}),  $C=5$ represents feature set (\texttt{\small stage\_color}, \texttt{\small stage\_friction}, \texttt{\small floor\_friction}, \texttt{\small block\_color}, \texttt{\small block\_mass}).
\begin{table}[!t]
\renewcommand\arraystretch{1.4}
    \centering
        \caption{Results for under different choice of $c$ on the \textsc{CausalWorld} benchmark (out-of-distribution). We report the average return of episodes (length varies for different tasks) over five random seeds.}\label{tab:c_change}
    \resizebox{\linewidth}{!}{
    \begin{tabular}{l|cccccc}\toprule
\texttt{Task Name}                  & $C=1$ & $C=2$ & $C=3$ & $C=4$ & $C=5$\\ \midrule 
 \textsc{Reaching}               & 928.62 $\pm$ 22.38  & 957.54 $\pm$ 18.39  & 976.60 $\pm$ 20.13 & 985.25 $\pm$ 17.26    & 1037.12 $\pm$ 19.15   \\ 
\textsc{Pushing}               & 389.16 $\pm$ 9.43  & 396.52 $\pm$ 17.29 & 405.08 $\pm$ 24.03 & 426.60 $\pm$ 15.37 & 429.42 $\pm$ 12.28  \\
\textsc{Picking}                 & 462.54 $\pm$ 9.08  & 484.21 $\pm$ 11.37  & 491.09 $\pm$ 6.44   & 522.96 $\pm$ 13.27 & 525.20 $\pm$ 12.28  \\
\textsc{Pick and Place}          & 464.68 $\pm$ 10.27  & 486.74 $\pm$ 8.52  & 490.24 $\pm$ 11.69 & 511.76 $\pm$ 9.05 & 523.46 $\pm$ 15.42 \\
\textsc{Stacking2}               & 794.81 $\pm$ 16.50  & 803.27 $\pm$ 13.26 & 831.82 $\pm$ 11.78 & 867.43 $\pm$ 9.82 & 871.43 $\pm$ 18.19  \\
Towers                  & 972.34 $\pm$ 12.36  & 979.23 $\pm$ 8.72 & 994.82 $\pm$ 5.76 & 1027.16 $\pm$ 17.25 & 1029.37 $\pm$ 8.06  \\
\textsc{Stacked Blocks}          & 2317.48 $\pm$ 74.32 & 2558.35 $\pm$ 42.17 & 2617.71 $\pm$ 88.07 & 2682.76 $\pm$ 69.25 & 2754.39 $\pm$ 82.16 \\
\textsc{Creative Stacked Blocks} & 1226.72 $\pm$ 62.18 & 1297.20 $\pm$ 39.42 & 1348.49 $\pm$ 55.05 & 1468.65 $\pm$ 27.63 & 1486.51 $\pm$ 41.29 \\
\textsc{General}                 & 868.62 $\pm$ 7.65  & 875.55 $\pm$ 19.28 & 891.14 $\pm$ 23.12 & 926.19 $\pm$ 17.34 & 934.74 $\pm$ 16.20 \\ \bottomrule
    \end{tabular}}
\end{table}

From the above Table~\ref{tab:c_change}, we can find that our OILCA can achieve a consistent performance improvement over different baselines under different choices of $c$. This demonstrates that the empirical performance of our method is relatively robust to the selection of this variable $c$. In fact, when increasing the number of intervened features (the number of $c$ choices), we can observe our model can achieve better performance. This is because the policy can learn to adapt to more diverse/uncertain environment configurations during the training phase.

\subsection{Influence of the augmented data}

In order to prove that the performance will not decay when further improving the $D_E/D_U$, we further increase $D_E/D_U$ (larger than 1) and conduct the experiments with three tasks in \textsc{Deepmind Control Suite} of our method OILCA. Moreover, to show the quality of augmented data, we show the performance gap when increasing expert data proportion using two kinds of augmented data: 1) sampling with the policy in the online environment for more true expert data (Expert Data), 2) our counterfactual data augmentation method OILCA (Augmented Data). The experimental results are shown in Table \ref{tab:expert_test}.

\begin{table}[h]
\centering
\caption{Results for the Influence of the augmented data with improving the proportion of augmented data and comparison to the true expert data in \textsc{DeepMind Control Suite} Benchmark.}\label{tab:expert_test}
\resizebox{\linewidth}{!}{
\begin{tabular}{c|cc|cc|cc}
\toprule  \texttt{Task Name} & \multicolumn{2}{c|}{\textsc{Cartpole Swingup}} & \multicolumn{2}{c|}{\textsc{Cheetah Run}} & \multicolumn{2}{c}{\textsc{Cartpole Swingup}} \\
\midrule \textbf{Proportion} & Augmented Data & Expert Data & Augmented Data & Expert Data & Augmented Data & Expert Data \\
\midrule 10\% & 430.21 $\pm$ 13.20 & 441.36 $\pm$ 12.01 & 71.85 $\pm$ 8.26 & 74.56 $\pm$ 3.29 & 261.77 $\pm$ 14.68 & 255.62 $\pm$ 18.29 \\
 30\% & 463.78 $\pm$ 21.95 & 472.92 $\pm$ 7.62 & 86.44 $\pm$ 13.62 & 82.06 $\pm$ 9.36 & 269.85 $\pm$ 13.39 & 272.18 $\pm$ 12.25\\
 50\% & 502.81 $\pm$ 20.76 & 520.15 $\pm$ 15.43 & 92.60 $\pm$ 16.51 & 89.21 $\pm$ 12.98 & 276.12 $\pm$ 9.82 & 285.48 $\pm$ 8.36\\
 70\% & 557.90 $\pm$ 16.62 & 562.89 $\pm$ 20.47 & 105.57 $\pm$ 11.29 & 111.27 $\pm$ 11.56 & 283.69 $\pm$ 12.71 & 295.83 $\pm$ 13.48 \\
 90\% & 589.01 $\pm$ 38.29 & 593.37 $\pm$ 16.81 & 113.12 $\pm$ 9.25 & 118.32 $\pm$ 15.27 & 288.27 $\pm$ 7.09 & 306.26 $\pm$ 10.81\\
 100\% & 608.38 $\pm$ 35.54 & 621.80 $\pm$ 9.26 & 116.05 $\pm$ 14.65 & 128.07 $\pm$ 8.31 & 298.73 $\pm$ 5.11 & 303.51 $\pm$ 11.67\\
 200\% & 596.52 $\pm$ 28.37 & 634.12 $\pm$ 18.29& 106.39 $\pm$ 10.08 & 132.64 $\pm$ 14.24 & 303.64 $\pm$ 12.91 & 311.70 $\pm$ 9.74 \\
 300\% & 612.30 $\pm$ 41.25 & 635.93 $\pm$ 25.15 & 118.51 $\pm$ 15.72 & 125.18 $\pm$ 8.73 & 301.57 $\pm$ 8.30 & 305.42 $\pm$ 14.53\\
 500\% & 601.47 $\pm$ 27.82 & 627.47 $\pm$ 22.86 & 109.96 $\pm$ 9.84 & 129.72 $\pm$ 12.34 & 289.15 $\pm$ 15.27 & 302.15 $\pm$ 12.16\\
 1000\% & 605.81 $\pm$ 31.63 & 629.94 $\pm$ 23.28 & 117.08 $\pm$ 7.69 & 124.80 $\pm$ 9.46 & 295.48 $\pm$ 7.84 & 304.93 $\pm$ 11.19\\
 \bottomrule
\end{tabular}}
\end{table}

From Table \ref{tab:expert_test}, we can find that the performance will converge when the proportion is close to 100\%, and further improving it indeed will not improve the performance obviously. This can be explained by the results that augmenting too much data can hardly bring additional effective information gain to the learned policy. Moreover, our augmented counterfactual data behaves slightly worse than augmentation with true expert data under most proportions, though achieving obvious improvement over other IL baselines. This shows that the augmented data through our method is high-quality enough. 

\subsection{Learning Curves of OILCA}
We provide the learning curves of OILCA in Figure~\ref{fig:learning_curves_1}. In detail, we deploy our trained policy to the online environment at each epoch and then collect 100 episodes for computing the average episode return. From the figures, we can observe that the policies can converge after 200 epochs in most tasks. The fluctuation of the curves mainly comes from the instability of the base offline IL method.

\begin{figure}[!t]
    \centering
    \subfigure[Cartpole Swingup]{\includegraphics[width=0.32\textwidth]{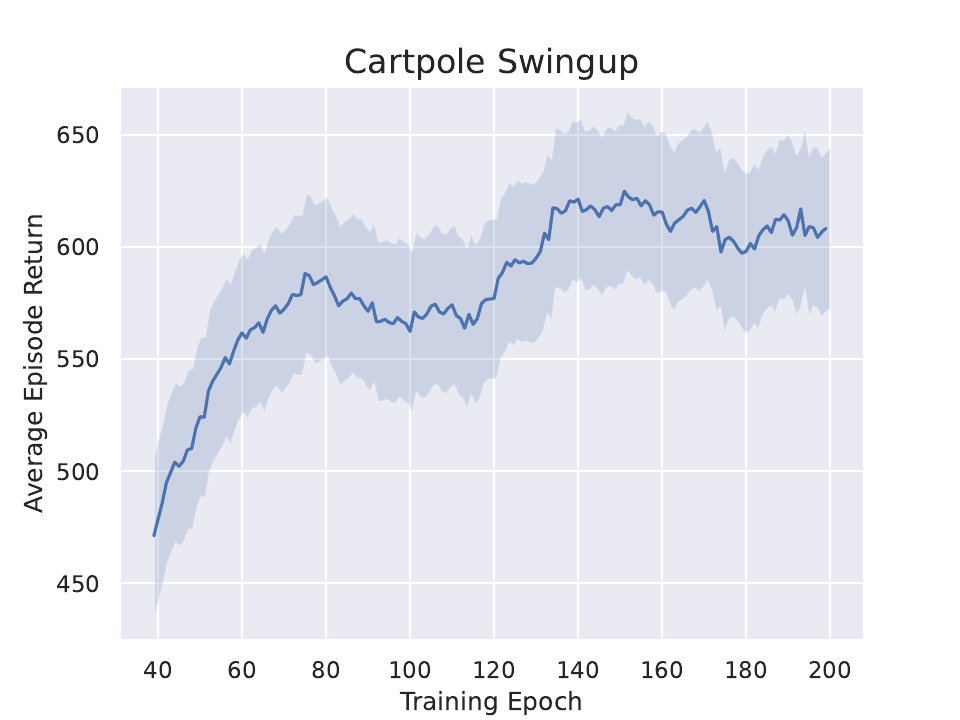}}
    \subfigure[Cheetah Run]{\includegraphics[width=0.32\textwidth]{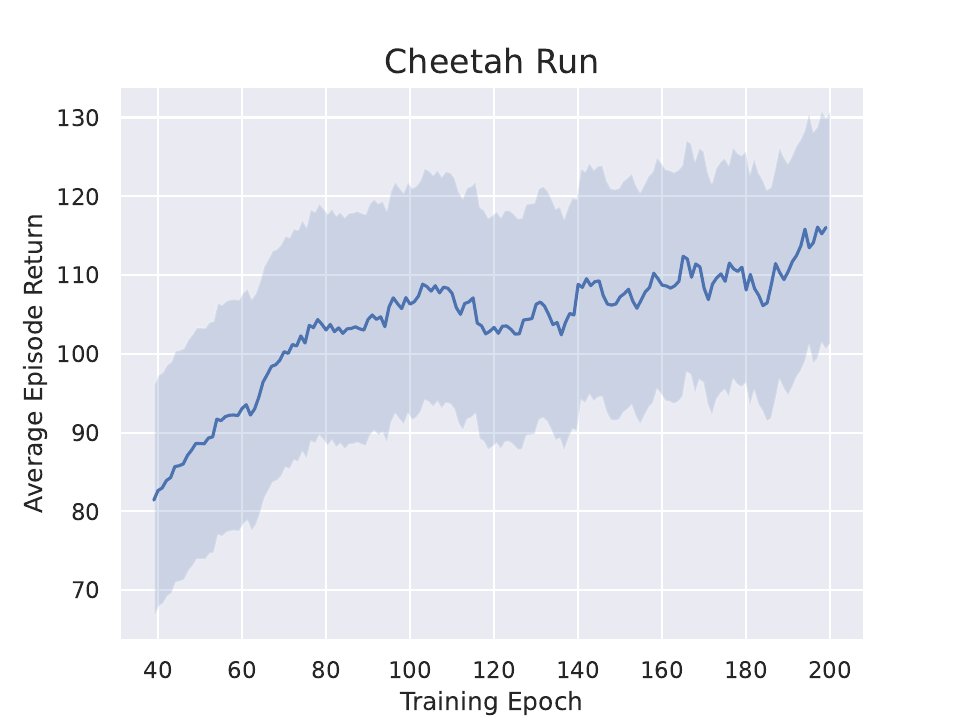}}
    \subfigure[Finger Turn Hard]{\includegraphics[width=0.32\textwidth]{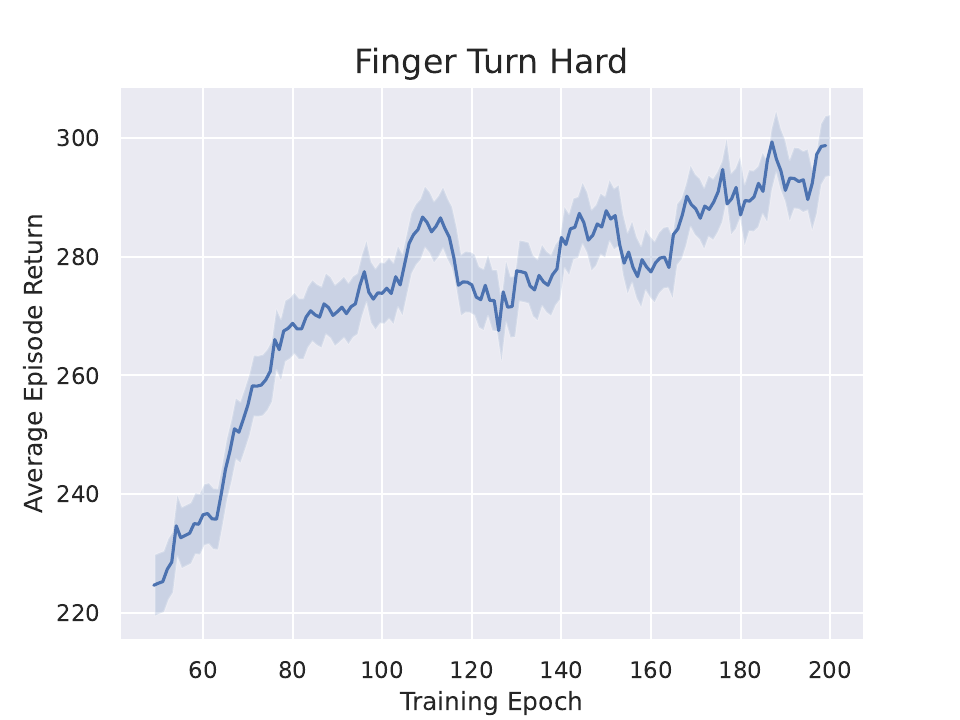}}\\
    \subfigure[Fish Swim]{\includegraphics[width=0.32\textwidth]{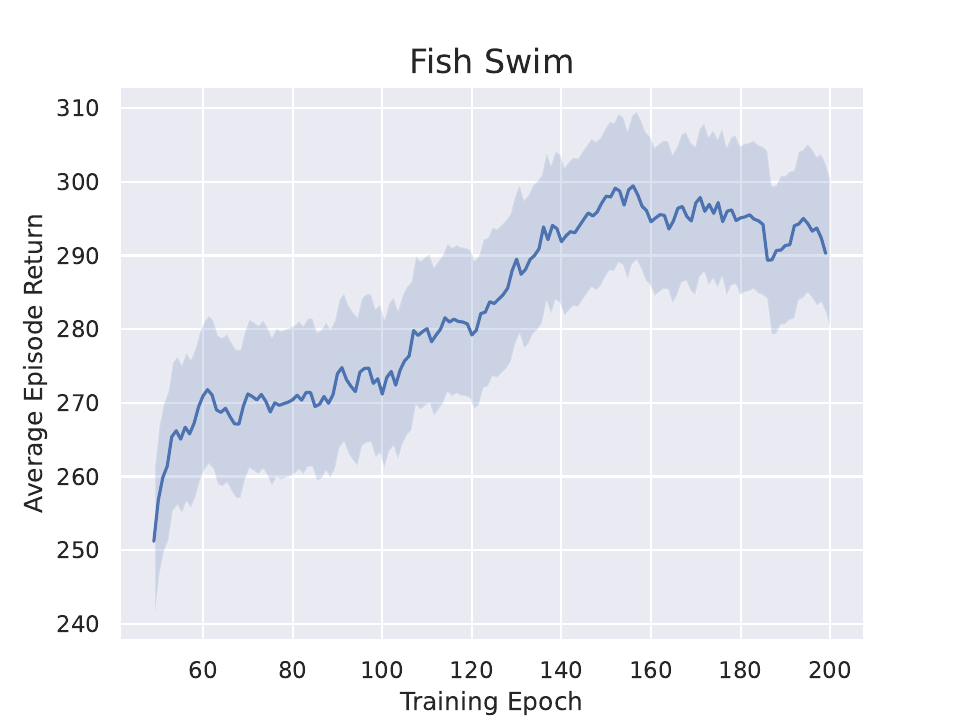}}
    \subfigure[Humanoid Run]{\includegraphics[width=0.32\textwidth]{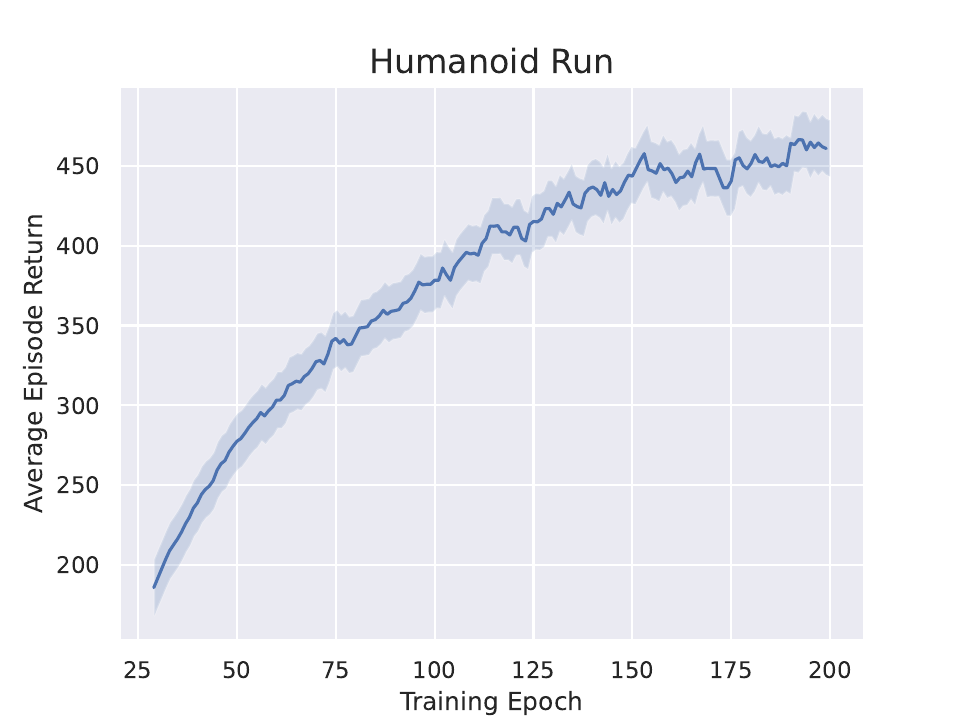}}
    \subfigure[Manipulator Insert Ball]{\includegraphics[width=0.32\textwidth]{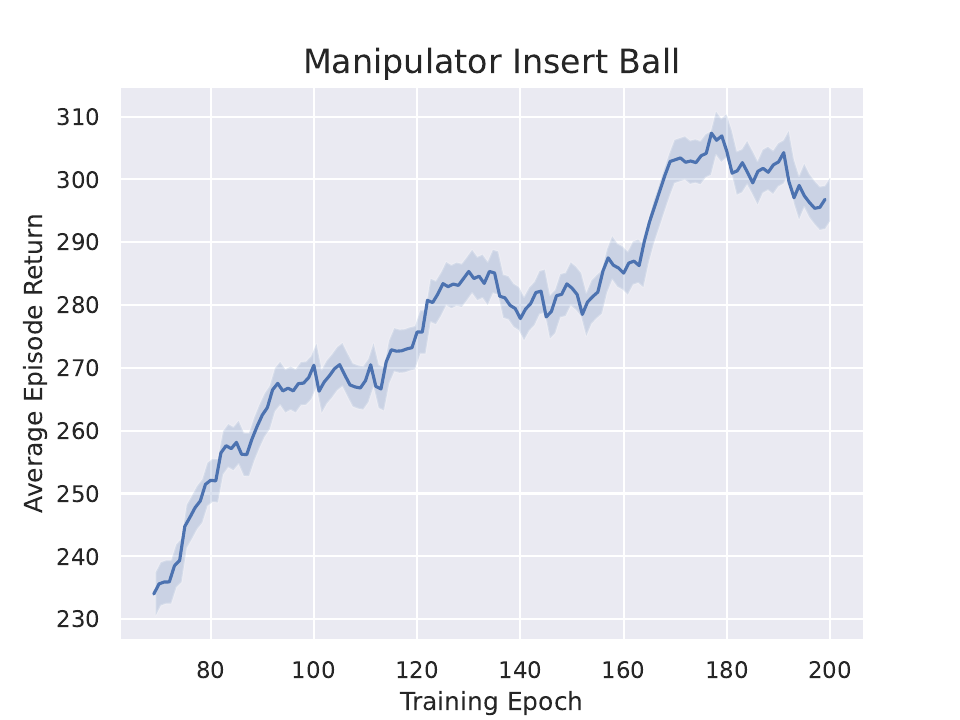}}
    \subfigure[Manipulator Insert Peg]{\includegraphics[width=0.32\textwidth]{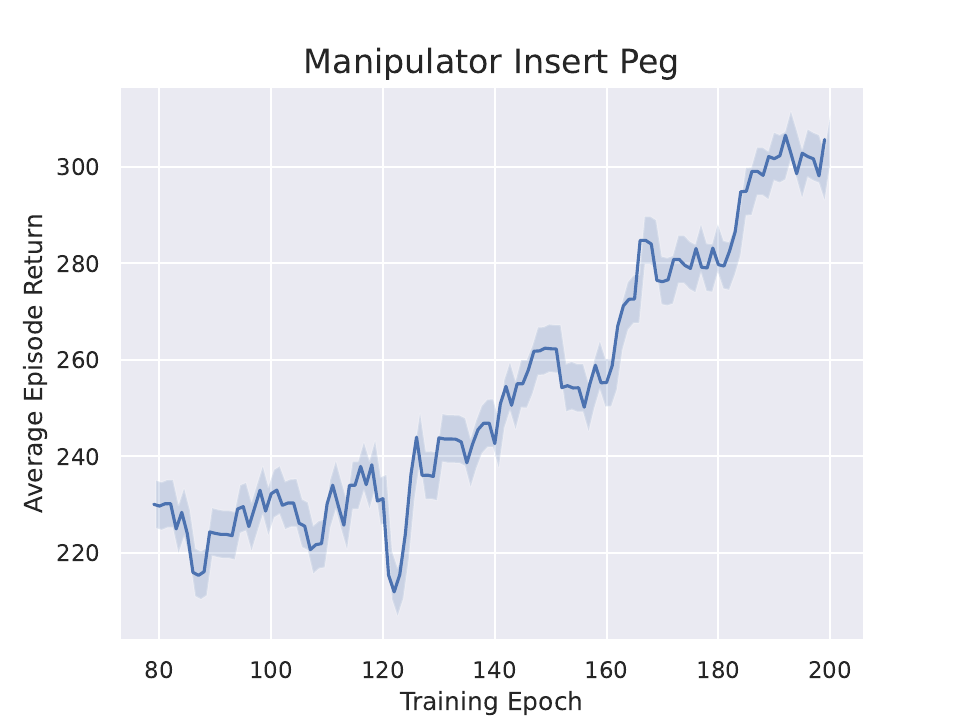}}
    \subfigure[Walker Stand]{\includegraphics[width=0.32\textwidth]{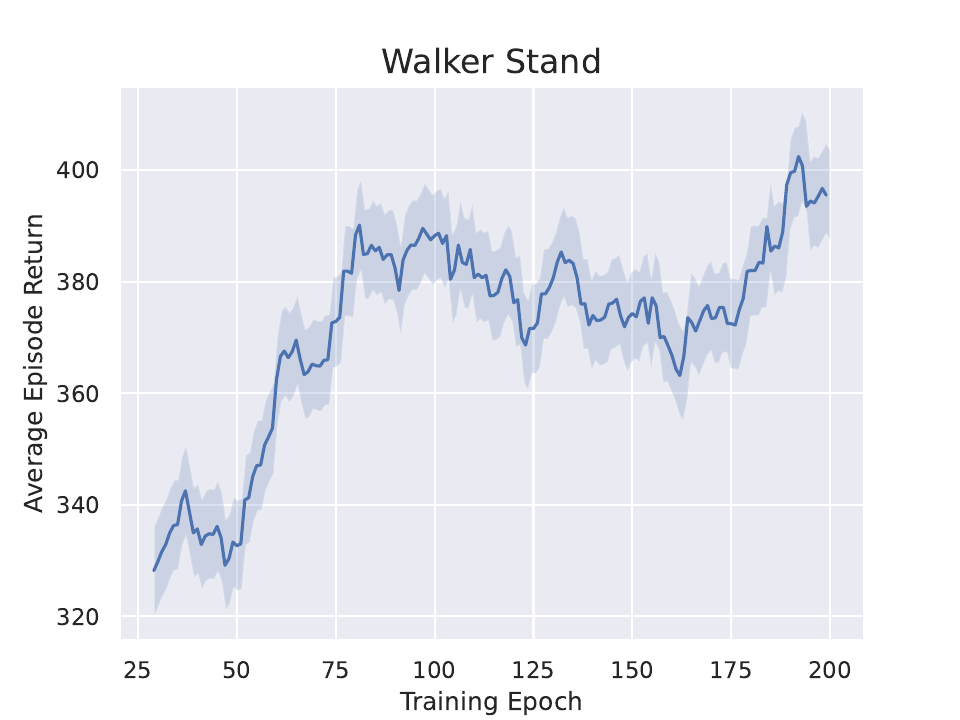}}
    \subfigure[Walker Walk]{\includegraphics[width=0.32\textwidth]{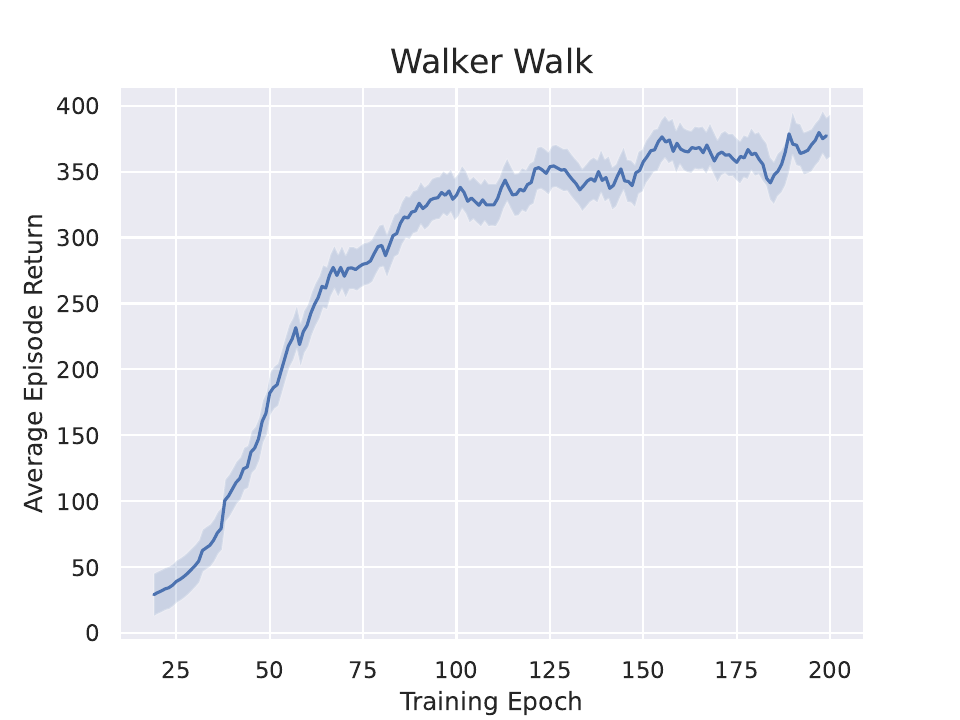}}
    \caption{Learning curves of OILCA on 9 tasks of \textsc{Deepmind Control Suite}.}
    \label{fig:learning_curves_1}
\end{figure}

\section{Limitation Analysis}
We simply analyze the limitations of this work in this section. In this paper, we only provide the theoretical guarantee to the generalization ability of learned policy from the perspective of the counterfactual samples' number. Actually, why the samples generated by the counterfactual augmentation module are more meaningful and can help the learned policy generalize better than samples obtained by other augmentation methods is also worth exploring theoretically.
\end{document}